\documentclass{article}

\usepackage[preprint,nonatbib]{neurips_2019}

\usepackage[utf8]{inputenc} 
\usepackage[T1]{fontenc}    
\usepackage{hyperref}       
\usepackage{url}            
\usepackage{booktabs}       
\usepackage{amsfonts}       
\usepackage{nicefrac}       
\usepackage{microtype}      
\usepackage[square,numbers]{natbib}

\usepackage{subcaption}     
\usepackage{amsmath, amsthm, thm-restate}
\usepackage{graphicx}
\usepackage{comment}
\usepackage{makecell}
\usepackage{tikz}
\usepackage{xcolor}
\usepackage{bbm}
\usepackage{sidecap}
\usepackage{fixltx2e}
\usepackage{enumitem}
\usepackage{makecell}
\usepackage{sidecap}
\usepackage{wrapfig}
\usepackage{multicol}

\title{Maximum Weighted Loss Discrepancy}

\author{Fereshte Khani, Aditi Raghunathan, Percy Liang\\
	Department of Computer Science, Stanford University \\
	\texttt{\{fereshte\texttt{},aditir\}@stanford.edu, pliang@cs.stanford.edu} \\
}

\usetikzlibrary{patterns}

\newcommand\sG{\ensuremath{\mathcal{G}}}

\newcommand\sO{\ensuremath{\mathcal{O}}}

\newcommand\sX{\ensuremath{\mathcal{X}}}
\newcommand\sY{\ensuremath{\mathcal{Y}}}

\newcommand{\var}{\text{Var}} 
\newcommand\p[1]{\ensuremath{\left( #1 \right)}} 
\newcommand\pb[1]{\ensuremath{\left[ #1 \right]}} 

\newcommand\half{\ensuremath{\frac{1}{2}}}
\newcommand\eqdef{\ensuremath{\stackrel{\rm def}{=}}} 
\newcommand{\1}{\mathbb{I}} 
\newcommand\refeqn[1]{(\ref{eqn:#1})}
\newcommand\refeqns[2]{(\ref{eqn:#1}) and (\ref{eqn:#2})}

\newcommand\refsec[1]{Section~\ref{sec:#1}}

\newcommand\reffig[1]{Figure~\ref{fig:#1}}

\newcommand\reftab[1]{Table~\ref{tab:#1}}

\newcommand\refthm[1]{Theorem~\ref{thm:#1}}

\newcommand\reflem[1]{Lemma~\ref{lem:#1}}

\newcommand\refprop[1]{Proposition~\ref{prop:#1}}

\newcommand\refdef[1]{Definition~\ref{def:#1}}
\newcommand\refcor[1]{Corollary~\ref{cor:#1}}

\ifthenelse{\isundefined{\definition}}{\newtheorem{definition}{Definition}}{}
\ifthenelse{\isundefined{\assumption}}{\newtheorem{assumption}{Assumption}}{}
\ifthenelse{\isundefined{\hypothesis}}{\newtheorem{hypothesis}{Hypothesis}}{}
\ifthenelse{\isundefined{\proposition}}{\newtheorem{proposition}{Proposition}}{}
\ifthenelse{\isundefined{\theorem}}{\newtheorem{theorem}{Theorem}}{}
\ifthenelse{\isundefined{\lemma}}{\newtheorem{lemma}{Lemma}}{}
\ifthenelse{\isundefined{\corollary}}{\newtheorem{corollary}{Corollary}}{}
\ifthenelse{\isundefined{\alg}}{\newtheorem{alg}{Algorithm}}{}
\ifthenelse{\isundefined{\example}}{\newtheorem{example}{Example}}{}
\newcommand{\E}{\ensuremath{\mathbb{E}}} 

\theoremstyle{definition}
\newtheorem{remark}{Remark}

\newenvironment{sproof}{%
	\proof}{\endproof}

\newcommand{\pr}{\ensuremath{\mathbb{P}}}

\DeclareMathOperator*{\argmax}{arg\,max}
\newcommand{\pab}[1]{\ensuremath{\left|#1\right|}}
\newcommand{\dfp}{D\textsubscript{FPR}}
\newcommand{\dfn}{D\textsubscript{FNR}}

\newcommand{\ws}{\ensuremath{w^{1/2}}}
\newcommand{\wsa}{\ensuremath{w^{1/2}_\text{A}}}
\newcommand{\weq}{\ensuremath{w^{0}}}
\newcommand{\wk}{\ensuremath{w^{k}}}

\newcommand{\clv}{CLV}
\newcommand{\lv}{LV}
\newcommand{\lr}{LR}
\newcommand{\mwld}{\text{MWLD}}

\begin{document}

\maketitle

\begin{abstract}
Though machine learning algorithms excel at minimizing the average loss over a population,
  this might lead to large discrepancies between the losses across groups within the population.
To capture this inequality, we introduce and study a notion we call maximum weighted loss discrepancy (\mwld{}),
  the maximum (weighted) difference between the loss of a group and the loss of the population.
We relate \mwld{} to group fairness notions and robustness to demographic shifts.
We then show \mwld{} satisfies the following three properties:
1) It is statistically impossible to estimate \mwld{} when all groups have equal weights.
2) For a particular family of weighting functions, we can estimate \mwld{} efficiently.
3) \mwld{} is related to loss variance, a quantity that arises in generalization bounds.
We estimate \mwld{} with different weighting functions on four common datasets from the fairness literature.
We finally show that loss variance regularization can halve the loss variance of a classifier
and hence reduce \mwld{} without suffering a significant drop in accuracy.
%
\end{abstract}
\section{Introduction}

Machine learning algorithms have a profound effect on people, especially around critical decisions such as banking and criminal justice \citep{berk2012criminal,barocas2016}.
It has been shown that standard learning procedures (empirical risk minimization)
can result in classifiers where some demographic groups suffer significantly larger losses than the average population
\citep{angwin2016machine,bolukbasi2016man}.
In this work, we consider the setting where demographic information is unavailable \cite{hashimoto2018repeated}, 
so we would like to ensure that no group suffers a loss much larger than average.

We are interested in measuring the maximum weighted loss discrepancy (\mwld) of a model,
which is, over all groups $g$, the maximum difference between the loss of a group $\E[\ell \mid g = 1]$ and the population loss $\E[\ell]$,
weighted by a function $w$ that quantifies the importance of each group:
\begin{align}
\mwld(w) = \max_{g} w(g) \pab {\E [\ell\mid g=1] - \E [\ell]}.
\end{align}
$\mwld$ captures various notions of group fairness;
for example, equal opportunity \citep{hardt2016} is \mwld{} with $\ell$ capturing false positives and weighting function $w(g) = 1$
for sensitive groups (e.g., defined by race, gender) and $0$ for all other groups.
We also show that we can bound the loss of a population with shifted demographics:
if we tilt the original distribution toward any group $g$ based on $w(g)$, the loss on the new distribution can bounded using $\mwld(w)$.

We consider estimating $\mwld$ from finite data by plugging in the empirical distribution.
There are two considerations:
(i) does the estimator converge? and (ii) can we compute the estimator efficiently?
The answers to these two questions depend on the weighting function.
We first show that for the uniform weighting function ($w^0(g) = \1[\E [g] > 0]$), we cannot estimate $\mwld(w^0)$ from finite samples (\refprop{impossibility}).
Next, we study a family of decaying weighting functions ($w^k(g) = \E[g]^k$), where
$k$ governs how much we account for the loss discrepancy of \textit{small} groups.
For this family, we show that the plug-in estimator (i) is efficient to compute and (ii) converges to the population $\mwld(w^k)$ (\refthm{auditor}).

Next, we show a connection to loss variance (\refprop{groupToVar}), an important quantity that arises in generalization bounds \citep{maurer2009empirical}
and is used as a regularization scheme \citep{mnih2008empirical,audibert2009exploration,shivaswamy2010empirical,namkoong2017variance}.
In particular, $\mwld(\ws)$ provides us with lower and upper bounds for the loss variance.
We also propose an extension called \emph{coarse loss variance}, which considers only a set of sensitive groups,
allowing us to incorporate knowledge about sensitive attributes.
%

We validate maximum weighted loss discrepancy on four common datasets from the fairness literature: predicting recidivism, credit rating, income, and crime rate.
We fit a logistic regression on these datasets and estimate its $\mwld(w^k)$ for various $k$.
We observe that \mwld{} with smaller $k$ converges more slowly to the population quantity. 
We also observe the group attaining \mwld{} shrinks as $k$ decreases.
We then use loss variance (LV) and coarse loss variance (CLV) regularization to train models.
Our empirical findings are as follows:
1) We halve the loss variance with only a small increase in the average loss.
2) In some cases, using loss variance as a regularizer simultaneously reduces the classification loss (higher accuracy) and loss variance (lower loss discrepancy).

\paragraph{Setup.} Consider the prediction task of mapping each input $x \in \sX$ to a probability distribution over an output space $\sY$. 
Let $h$ be a predictor which maps each input $x$ to a probability distribution over $\sY$;
for binary classification, $\sY = \{0, 1\}$ and a predictor $h : \sX \to [0, 1]$ returns the probability of $y = 1$.
Let $\ell(h, z)$ be the (bounded) loss incurred by predictor $h$ on individual $z$---e.g., the zero-one or logistic loss.
Let $p^\star$ denote the underlying distribution over individuals $z = (x, y)$;
all expectations are computed with respect to $z \sim p^\star$.
Define a \emph{group} to be a measurable function $g: \sX \times \sY \to \{0, 1\}$ such that $g(z) = 1$ if individual $z$ is in the group and $0$ otherwise.
Let $\sG$ be the set of all groups.
When clear from context, we use $\ell$ to denote $\ell(h, z)$ and $g$ to denote $g(z)$,
so that $\E[\ell]$ is the population loss and $\E[\ell \mid g = 1]$ is the loss of group $g$.

\section{Maximum Weighted Loss Discrepancy}
\label{sec:mwld}
We now introduce our central object of study:
\begin{definition}[Maximum weighted loss discrepancy (\mwld)]
	\label{def:mwld}
  For a weighting function $w : \sG \to [0, 1]$, loss function $\ell$, and predictor $h$,
  define the maximum weighted loss discrepancy $\mwld(w, \ell, h)$
  to be the maximum difference between the loss of a group $g$ and the population loss, weighted by $w(g)$:
	\begin{align}
	\label{eqn:mwld}
	\mwld\left(w, \ell, h\right) &\eqdef \sup \limits_{g \in \sG} w(g) \left| \E[\ell \mid g = 1] -  \E[\ell] \right|,
	\end{align}
\end{definition}
where the weighting function $w$ (e.g., $w(g) = \E[g]^{1/2}$) intuitively controls the importance of group $g$.

\paragraph{Group fairness interpretation.} By rearranging the terms of \refeqn{mwld},
we can bound the loss discrepancy of any group in terms of the group weight and the maximum weighted loss discrepancy:
\begin{align}
\label{eqn:guarantee}
\pab {\E[\ell \mid g = 1] - \E [\ell] } \leq  \frac{\mwld(w, \ell, h)}{w(g)},
\end{align}
where the bound is tighter for larger $w(g)$.
Existing statistical notions of fairness such as equal opportunity \citep{hardt2016} can be viewed as enforcing $\mwld$ to be small
for a weighting function $w$ that is $1$ on sensitive groups (e.g., different races) and $0$ on all other groups;
see Appendix~\ref{sec:previous_statistical_notions} for further discussion.

\paragraph{Distributional shift interpretation.}
We can use $\mwld$ to bound the loss on a population with shifted demographics.
For any group $g$, define the mixture distribution $q(\cdot) \eqdef w(g) p^\star(\cdot \mid g = 1) + (1 - w(g)) p^\star(\cdot \mid g = 0)$,
which tilts the original distribution $p^\star$ more towards group $g$ (assuming $w(g) \ge \E[g]$).
Then via simple algebra (\refprop{robustness} in Appendix~\ref{sec:appendix}),
the loss under this new distribution $q$ can be controlled as follows:
\begin{align}
\label{eqn:demographic_shift_guarantee}
\E_{z \sim q}[\ell] \le \E_{z \sim p^\star} [\ell] + \mwld(w, \ell, h).
\end{align}
This is similar in spirit to distributionally robust optimization (DRO) using a max-norm metric \citep{duchi2018learning},
but the difference is that the mixture coefficient is group-dependent.


How do we now operationalize $\mwld$?
After all, the supremum over all groups $g$ in $\mwld$ \refeqn{mwld} appears daunting.
In \refsec{estimating_mwld}, we show how we can efficiently estimate $\mwld$ for a restricted family of weighting functions.

\section{Estimating Maximum Weighted Loss Discrepancy (\mwld)}
\label{sec:estimating_mwld}

We now focus on the problem of estimating $\mwld$ from data.
For simplicity of notation, we write $\mwld(w)$ instead of $\mwld(w,\ell,h)$.
Given $n$ points $z_1, \dots, z_n \sim p^\star$, our goal is to derive an estimator
 $\widehat\mwld_n(w)$ that is (i) efficient to compute and (ii) accurately approximates $\mwld(w)$.
 Formally:
 \begin{align}
 \label{eqn:estimating}
   \forall \epsilon, \delta \in (0,\half): \quad \pr\pb{\pab {\mwld(w) - \widehat \mwld_n(w)} \ge \epsilon} \le \delta \text{\quad and } n=\text{poly}(\log(1/\delta), 1/\epsilon).
 \end{align}
Whether this goal is achievable depends on the weighting function.
Our first result is that we cannot estimate $\mwld$ for the uniform weighting function ($\weq \eqdef \1 [\E [g] > 0]$):

\begin{restatable}{proposition}{impossibility}
	\label{prop:impossibility}
	For any loss function $0\le \ell\le 1$ and predictor $h$,
  if $(\ell, h)$ is non-degenerate ($\min_z \ell(h,z) = 0$ and $\max_z \ell(h,z) = 1$),
  then there is no estimator $\widehat\mwld_n(\weq, \ell, h)$ satisfying \refeqn{estimating}.
\end{restatable}
We prove \refprop{impossibility} by constructing two statistically indistinguishable distributions such that $\mwld(\weq) \ge \half$ for one and $\mwld(\weq) < \half$ for the other (see Appendix~\ref{sec:appendix} for details).
\refprop{impossibility} is intuitive since $\mwld$ for $w(g) = 1$ is asking for uniform convergence over all measurable functions,
which is a statistical impossibility.
It therefore seems natural to shift our focus to weighting functions $w$ that decay to zero as the measure of the group goes to zero. 

As our next result, we show that we can estimate $\mwld$ for the weighting function $\wk(g) \eqdef \E [g]^k$ for $k \in (0,1]$.
In particular, we show (i) we can efficiently compute the empirical $\widehat\mwld_n(\wk)$, and (ii) it converges to $\mwld(\wk)$.
Letting $\widehat\E[\cdot]$ denote an expectation with respect to the $n$ points, define the plug-in estimator:
\begin{align}
\label{eqn:empirical_mwld}
\widehat\mwld_n(\wk) &\eqdef \max \limits_{g \in \sG} \widehat \E [g]^k  \Big | \widehat\E[\ell \mid g = 1] - \widehat\E[\ell] \Big|.
\end{align}
Although $\widehat \mwld_n(\wk)$ seems to have an intractable max,
in the next theorem we prove we can actually compute it efficiently,
and that $\widehat \mwld_n(\wk)$ converges to $\mwld(\wk)$, where the rate of convergence depends on $k$.
The key is that $\widehat \mwld_n(\wk)$ attains its max over a linear number of possible groups based on sorting by losses;
this is what enables computational efficiency as well as uniform convergence.

\begin{restatable}{theorem}{EfficientAuditor}
	\label{thm:auditor}
	For $k \in (0,1]$, let $\wk \eqdef \E [g]^k$.
	Given $n$ i.i.d. sample from $p^\star$,
	we can compute $\widehat\mwld (\wk)$\refeqn{empirical_mwld} efficiently in $\sO(n \log n)$ time; and 
	for any parameters $\delta, \epsilon \in (0, \half)$, for a constant $C$, if $n\ge\frac{C\log (1/\delta)}{\epsilon^{2+\frac{2}{k}}}$, the following holds:
$\pr \pb{\pab{\mwld(\wk) - \widehat{\mwld}_n(\wk)} \ge \epsilon} \le \delta$.
\end{restatable}
\begin{sproof}
For computational efficiency, we show that if we sort the $n$ points by their losses $\ell_1\ge \ell_2\ge \dots\ge \ell_n$ (in $\sO(n \log n)$ time),
  then there exists an index $t$ such that either $g = \{\ell_1, \dots, \ell_t\}$ or $g = \{\ell_{t+1}, \dots, \ell_n\}$ achieves the empirical maximum weighted loss discrepancy \refeqn{empirical_mwld}.

To show convergence, let $D(g)$ be the weighted loss discrepancy for group $g$, $D(g) \eqdef \E [g]^k \pab {\E [\ell \mid g=1]-\E [\ell]}$, and analogously let $\widehat D(g)$ be its empirical counterpart.
	We first prove for any $\alpha > 0$ if $\E [g] \ge \alpha$,  then $D(g) - \widehat D(g) \le C\sqrt {\frac{\log (1/\delta)}{n\alpha}}$ with probability at least $1-\delta$, for some constant $C$.
	Furthermore, since we assumed $\ell \le 1$, then $D(g) - \widehat D(g) \le \alpha^k$.
  By combining these two upper bounds ($\max_\alpha \min (\text {bound 1}, \text {bound 2})$), we compute an upper bound independent of $\alpha$, thereby applicable to all groups. 

  To show uniform convergence, from the sorting result, we only need to consider groups of the form $\1[\ell(h, z) \ge u]$ and their counterparts $\1[\ell(h, z) \le u]$. We prove uniform convergence over this set using the KKW inequality \citep{massart1990tight} and the same procedure we explained before for convergence. See Appendix~\ref{sec:appendix} for the complete proof.

\end{sproof}

%

\section{A Closer Look at $w(g)=\E [g]^k$ and Connection to Loss Variance}
\label{sec:weighting_meaning}

As shown in \refsec{mwld}, \mwld{} has two different interpretations: group fairness and distributional shift interpretations.
In this section, we look at these interpretations for the family of weighting functions, $\wk(g) = \E[g]^k$, for which we can efficiently estimate \mwld(\wk).
In \refsec{loss_variance}, we show a connection between a particular member of this family ($k = 1/2$) and \emph{loss variance}.
As an extension, in \refsec{fairness_CLV}, we introduce \emph{coarse loss variance}, a simple modification of loss variance which measures weighted loss discrepancy only for sensitive groups.


From the group fairness interpretation \refeqn{guarantee}, $\mwld(\wk)$ provides guarantees on the loss discrepancy of each group according to its size. Therefore groups with similar sizes have similar guarantees. For a fixed value of $\mwld(\wk)$ (here $0.1$), \reffig{different_k_guarantees} (left) shows the bounds on the group loss discrepancy for different sizes and $k$.
The upper bound guarantee for smaller groups is weaker, and the parameter $k$ governs how much this upper bound varies across group sizes.

From the distributional shift interpretation \refeqn{demographic_shift_guarantee}, $\mwld(\wk)$ provides a guarantee on the loss of a new distribution where the weight of $g$ is increased by a maximum factor of $\E[g]^k$. \reffig{different_k_guarantees} (right) illustrates this maximum upweighting factor $\E[g]^k$.  
The upweighting factor for smaller groups is smaller, and the parameter $k$ governs how much this factor varies across group sizes.

\begin{figure}[t]
	\centering
	\begin{minipage}[t]{0.64\textwidth}
		\centering
		\includegraphics[width=0.8\textwidth]{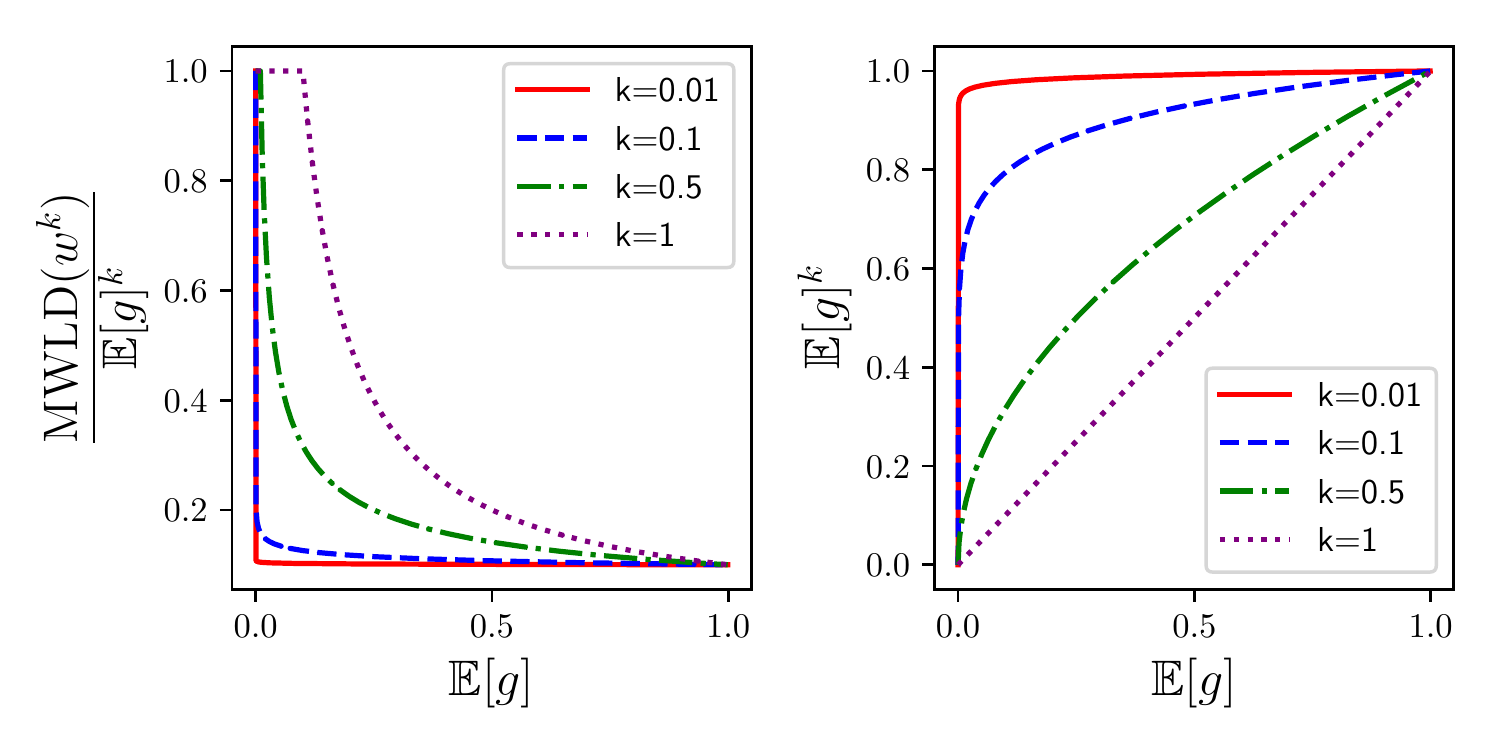}
			\caption{\label{fig:different_k_guarantees}
				left: Upper bound guarantee for loss discrepancy of a group for different values of $k$.
				right: Magnitude of the shift in a group is dictated by the weighting function of the group.
        }
	\end{minipage}\quad%
		\begin{minipage}[t]{0.32\textwidth}
			\centering
			\includegraphics[width=0.8\textwidth]{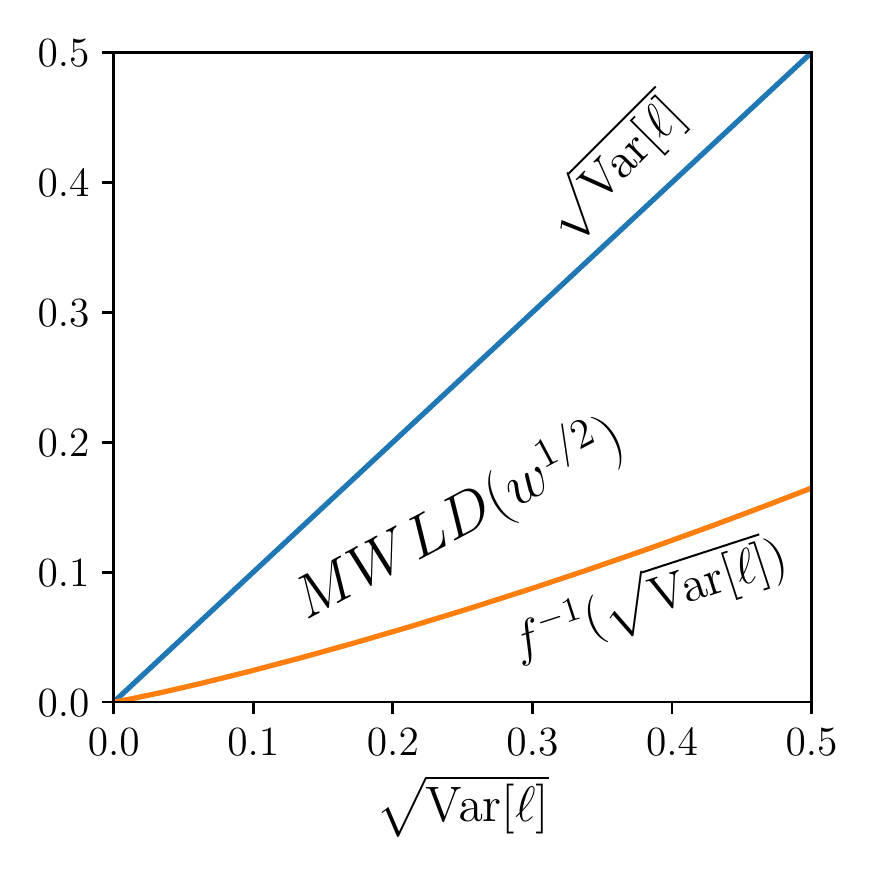}
			\caption{\label{fig:upper_lower_bounds}
				Relationship between $\var [\ell]$ and  $\mwld(\ws)$. Here $f(x) \eqdef x\sqrt{2-4\ln (x)}$.}
		\end{minipage}
\end{figure}

\subsection{Loss Variance and Maximum Weighted Loss Discrepancy}
\label{sec:loss_variance}
In this section, we show an interesting connection between a particular member of the introduced family of weighting function, $\ws(g) = \E[g]^{1/2}$,
and loss variance, which appears prominently in generalization bounds \citep{maurer2009empirical}.
Loss variance, $\var[\ell]$, 
is the average squared difference between the loss of individuals and the population loss:
\begin{align}
\label{eqn:averageIndividual}
  \var[\ell] =\E \pb{\p{\ell - \E [\ell]}^2}.
  \end{align}
From the law of total variance, we have $\var[\ell] \ge \var[\E[\ell \mid g]]$ for any group $g$.
By observing that $\var[\E[\ell \mid g]] \geq \E [g] \p{\E [\ell \mid g=1] - \E [\ell]}^2$,
we see that square root of loss variance is an upper bound on $\mwld(\ws)$.
This allows us to bound the loss of any group in terms of the loss variance (using \refeqn{guarantee}).
A natural next question is about the tightness of the upper bound. 
How much larger can the variance
be compared to the $\mwld(\ws)$?
The next proposition shows that loss variance also provides a lower bound on a function of the $\mwld(\ws)$. 
\begin{restatable}{proposition}{propGroupToVar}
	\label{prop:groupToVar}
	For any measurable loss function $0 \le \ell \le 1$, the following holds: $\mwld(\ws) \le \sqrt{\var[\ell]} \le \mwld(\ws)\sqrt{2-4\ln\left({\mwld(\ws)}\right)}$.
\end{restatable}
\begin{sproof}
%
We first center the losses and make the average loss $0$ (without changing the MWLD).
For any $u>0$, let $g_u$ be the group of points with loss greater than $u$.
By definition of $\mwld$, we have $\sqrt{\E [g_u]} u \le \sqrt{\E [g_u]} \E [\ell \mid g_u=1] \le \mwld(\ws)$.
Therefore we have $\pr [\ell \ge u] = \E [g_u] \le \frac{\mwld(\ws)^2}{u^2}$.
Using integration by parts, we express variance with an integral expression in term of cumulative density function (CDF).
Plugging this bound into the integral expression yields the result.  
For more details, see Appendix~\ref{sec:appendix}.
\end{sproof}

\reffig{upper_lower_bounds} shows the bounds on $\mwld(\ws)$ for different values of $\sqrt{\var [\ell]}$.
This proposition establishes a connection between statistical generalization and \mwld{} (and thereby group fairness).
Furthermore, this connections states that reducing loss discrepancy between individuals and the population \refeqn{averageIndividual} leads to lower loss discrepancy between groups and the population \refeqn{mwld} and vice versa.



\begin{SCfigure}	\includegraphics[width=0.55\textwidth]{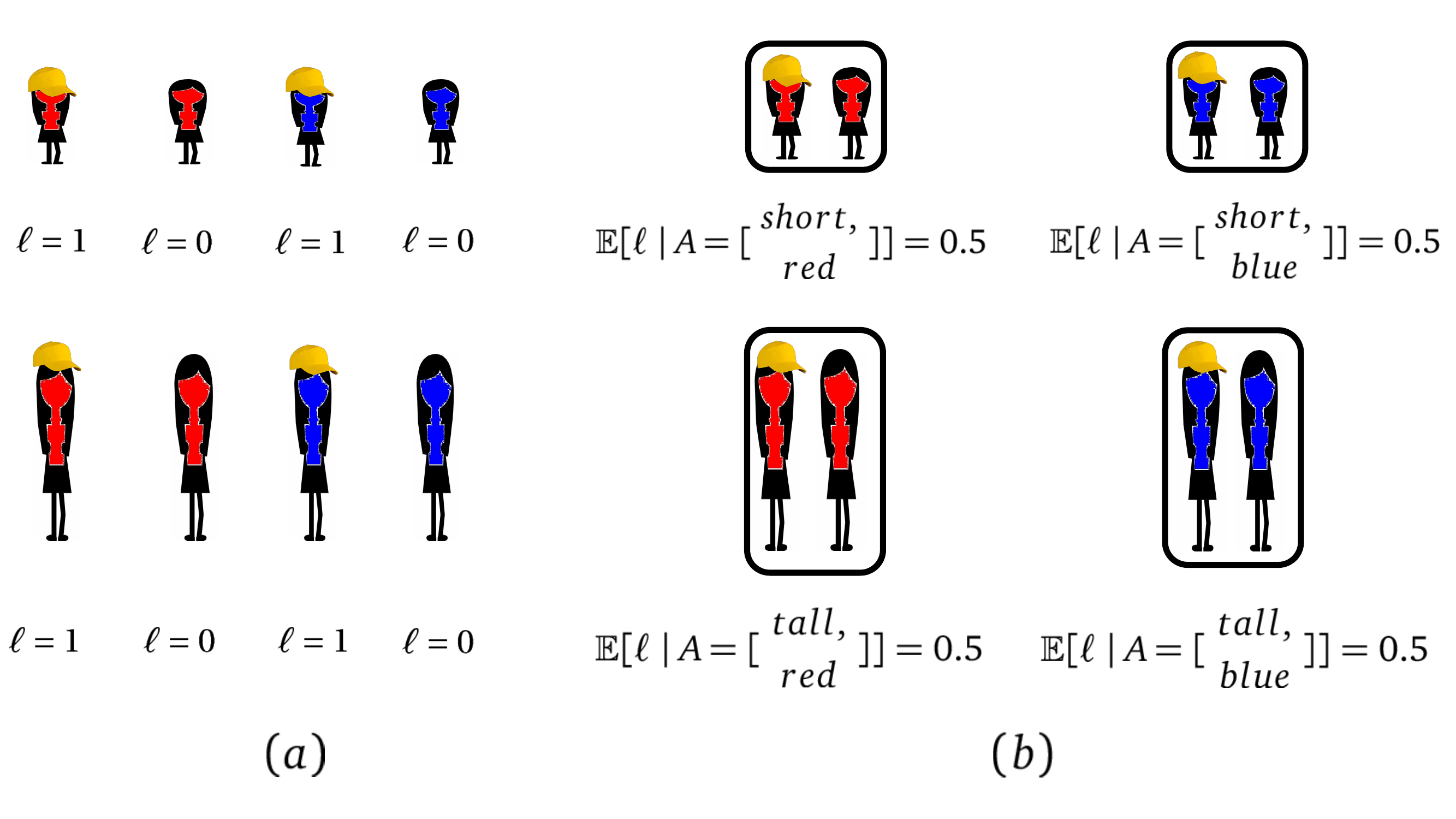}
	\caption{
		\label{fig:LV_vs_CLV}
		Each individual has two sensitive attributes, color and height, and one non-sensitive attribute, having a hat,
		(a) $\mwld(\ws) \le \sqrt{\var [\ell ]}= 0.5$.
		(b) $\mwld(\wsa) \le \var [\E [\ell \mid A]] = 0$. 
    Note that, since the weights of the groups not defined on sensitive attributes are $0$, their expected loss can deviate a lot from average loss (e.g, the expected loss of individuals with hats is $1$, which deviates a lot from average loss $\E [\ell]=0.5$).}
\end{SCfigure}

\subsection{Sensitive Attributes and Coarse Loss Variance}
\label{sec:fairness_CLV}
So far, we have focused on the loss discrepancy over all groups, which could be too demanding.
Suppose we are given a set of sensitive attributes (e.g., race, gender),
and we are interested only in groups defined in terms of those attributes.
We define \emph{coarse loss variance}, which first averages the losses of all individuals with the same sensitive attribute values and
considers the variance of these average losses.
Formally, let $A$ denote the sensitive attributes; for example, $A = [\text{race}, \text{gender}, \dots]$.
Then the coarse loss variance is:
\begin{align}
\label{eqn:coarseLossVariance}
\var \left[ \E[\ell \mid A ] \right] &= \E \left[\p{\E[\ell \mid A] - \E[ \ell] }^2 \right].
\end{align}

Coarse loss variance is  smaller than loss variance \refeqn{averageIndividual} because it ignores fluctuations in the losses of individuals who have identical sensitive attributes.
\reffig{LV_vs_CLV} shows the difference between loss variance and coarse loss variance. 
%
Analogous to \refprop{groupToVar} in previous section, we show that coarse loss variance is a close estimate of $\mwld(w)$ where $w(g)=\sqrt{\E [g]}$ if $g$ is a function only of sensitive attributes and $0$ otherwise.
Define $\sG_A$ to be the set of groups $g$ such that $g(z)$ only depends on the sensitive attributes $A(z)$.
Let $\wsa (g) \eqdef \1 [g \in \sG_A]\sqrt{\E [g]}$, then we have:
\begin{align}
\label{eqn:coarseUnfairness}
\mwld(\wsa) \leq \sqrt{\var[\E[\ell\mid A]]} \le \mwld(\wsa)\sqrt{2-4\ln\left({\mwld(\wsa)}\right)}
\end{align}
For the formal propositions regarding coarse loss variance, see \refcor{protectedGroup} in Appendix~\ref{sec:appendix}.
As a caveat, empirical coarse loss variance converges slower to its population
counterpart in comparison to loss variance (See \refthm{coarse_variance_generalizes} in Appendix~\ref{sec:generalization_bounds} for the exact convergence rate).




\begin{remark}
	\label{remark:loss_discussion}
  In some applications, the impact of misclassification could be quite different depending on the true label of the individual. 
  For example, in deciding whether to provide loans, denying an eligible individual a loan could have a greater negative 
  impact than providing a loan to a defaulter.
  In such situations, we consider the variance conditioned on the label, i.e., $\E[\var[\E [\ell \mid y,A]\mid y]]$, so that we do not attempt to pull the losses of individuals with different labels (and consequently different
  impacts of misclassification) together.
\end{remark}
 %

\section{Experiments}
\label{sec:experiments}
We first explore the effect of parameter $k$ in $\mwld(\wk)$, as discussed in \refsec{weighting_meaning}.
We then use (coarse) loss variance to train models and show that we can halve the loss variance without significant increase in average loss.
\reftab{dataset_statistics} shows a summary of the datasets considered. 
For more details about these datasets, see Appendix~\ref{app:dataset_appendix}.

\begin{table*}
	\centering
	\scalebox{0.7}{
	\hspace{-0.5cm}
	\begin{tabular}{lrrlll}
	\toprule
    \bf Name & \bf \# Records & \bf \# Attributes & \bf Variable to be predicted & \bf Sensitive attributes & \bf Other attributes \\ \midrule	
	\bf C\&C &1994&99&High or low crime rate& Different races percentage  (discretized)& Police budget, \#Homeless, \dots\\
	\bf Income & 48842 & 14 &High or low salary &Race, Age (discretized), Gender&Martial status, Occupation, \dots\\
	\bf German & 1000 & 22 & Good or bad credit rating & Age(discretized), Gender& Foreign worker, Credit history, \dots\\
	\bf COMPAS\_5 &7214 &5&Recidivated or not&Race, Age (discretized), Gender& Prior counts, Charge degree, \dots\\ \bottomrule
	\end{tabular}}
\caption{\label{tab:dataset_statistics}Statistics of datasets. For more details about these datasets, see Appendix~\ref{app:dataset_appendix}.}
\end{table*}

\subsection{Estimating Maximum Weighted Loss Discrepancy}
We first fit a logistic regression (LR) predictor on these datasets, with the following objective: $\sO_{\text{LR}} \eqdef \widehat\E[\ell] + \eta \|\theta\|_2^2$. 
\reffig{different_k_auditing}(a) shows the values of $\mwld(\wk)$ for different value of $k$.
As shown in \refthm{auditor} we expect $\mwld(\wk)$ to converge slower to the population for smaller $k$.
Empirically, we observe a bigger train-test gap for $\mwld(\wk)$ for smaller $k$.

As discussed in \refsec{weighting_meaning}, according to the group fairness interpretation of $\mwld(\wk)$, we can bound the loss of any group in term of $\mwld(\wk)$; 
where small $k$ leads to similar upper bound for all groups, while larger $k$ allows weaker upper bounds for smaller groups.
For each $\alpha$, we compute the maximum loss discrepancy for groups with size $\alpha$ in COMPAS\_5 dataset (i.e., $\sup_{g: \E [g]=\alpha} \pab{\E [\ell \mid g=1] - \E [\ell]}$). The solid black line in \reffig{different_k_auditing}(b) shows this plot.
For different values of $k$, we plot the obtained upper bound from $\mwld(\wk)$\refeqn{guarantee}.
Smaller $k$ leads to tighter upper bound for small groups and large $k$ leads to tighter upper bound for large groups. 
%

\begin{figure}
	\centering
			\begin{minipage}[t]{0.72\textwidth}
					\centering
	\includegraphics[height=110pt]{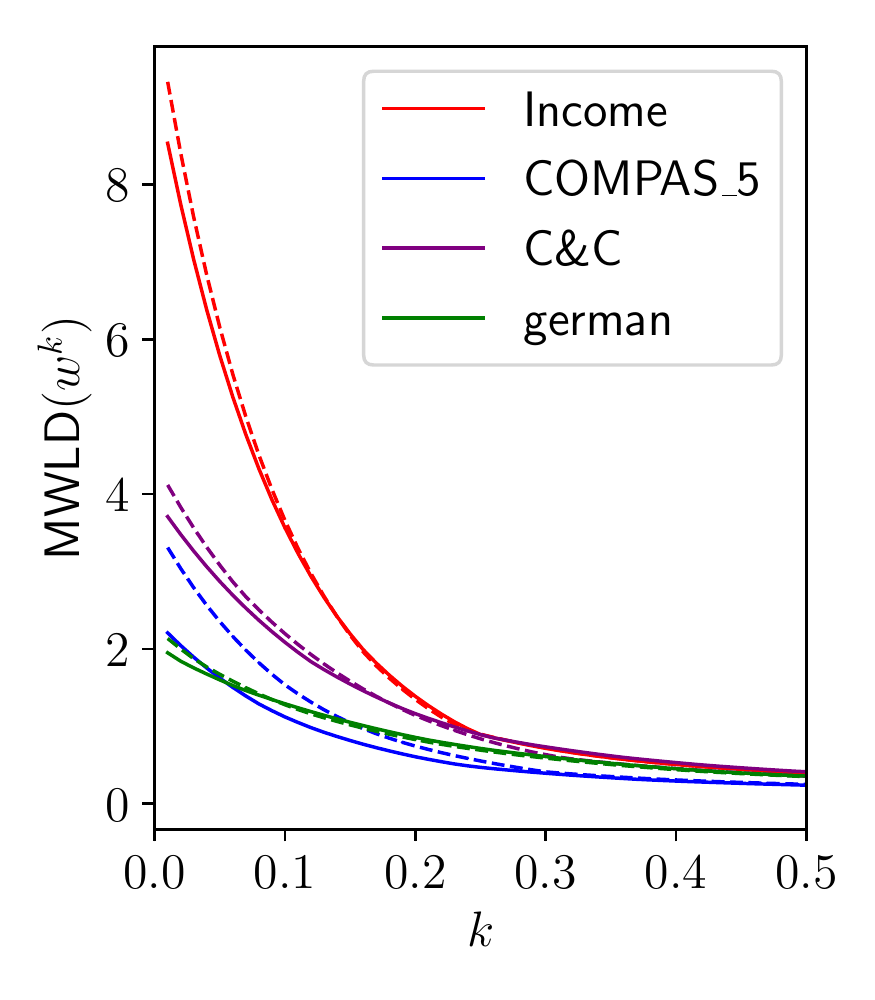}\hspace{0.1\textwidth}%
		\includegraphics[height=110pt]{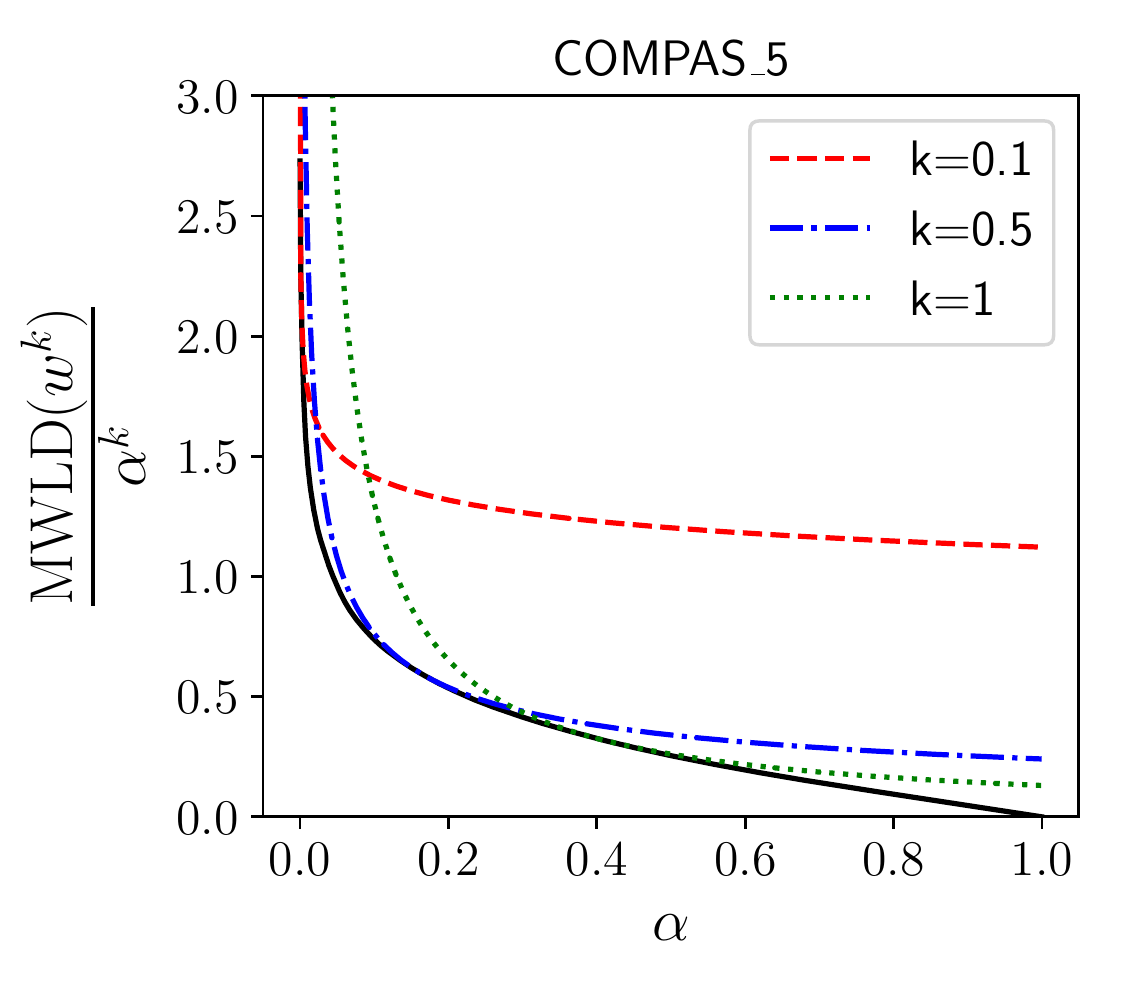}
	
\vspace{-0.2cm}
\tiny\hspace{0.01\textwidth}(a)\hspace{0.45\textwidth}(b)
				\caption{ \label{fig:different_k_auditing}
					(a) The gap between values of $\mwld(\wk)$ in train (dashed lines) and test (solid lines) is larger for smaller $k$. 
          (b) The solid black line indicates the maximum loss discrepancy for different group sizes. 
					Dashed lines show the obtained upper bound from $\mwld(\wk)$ 
					 \refeqn{guarantee}.
					 The upper bound is tighter for smaller groups when $k$ is small, and it is tighter for larger groups when $k$ is large.
	}
	\end{minipage}\quad%
\begin{minipage}[t]{0.23\textwidth}

\hspace{-0.8cm}\includegraphics[height=110pt]{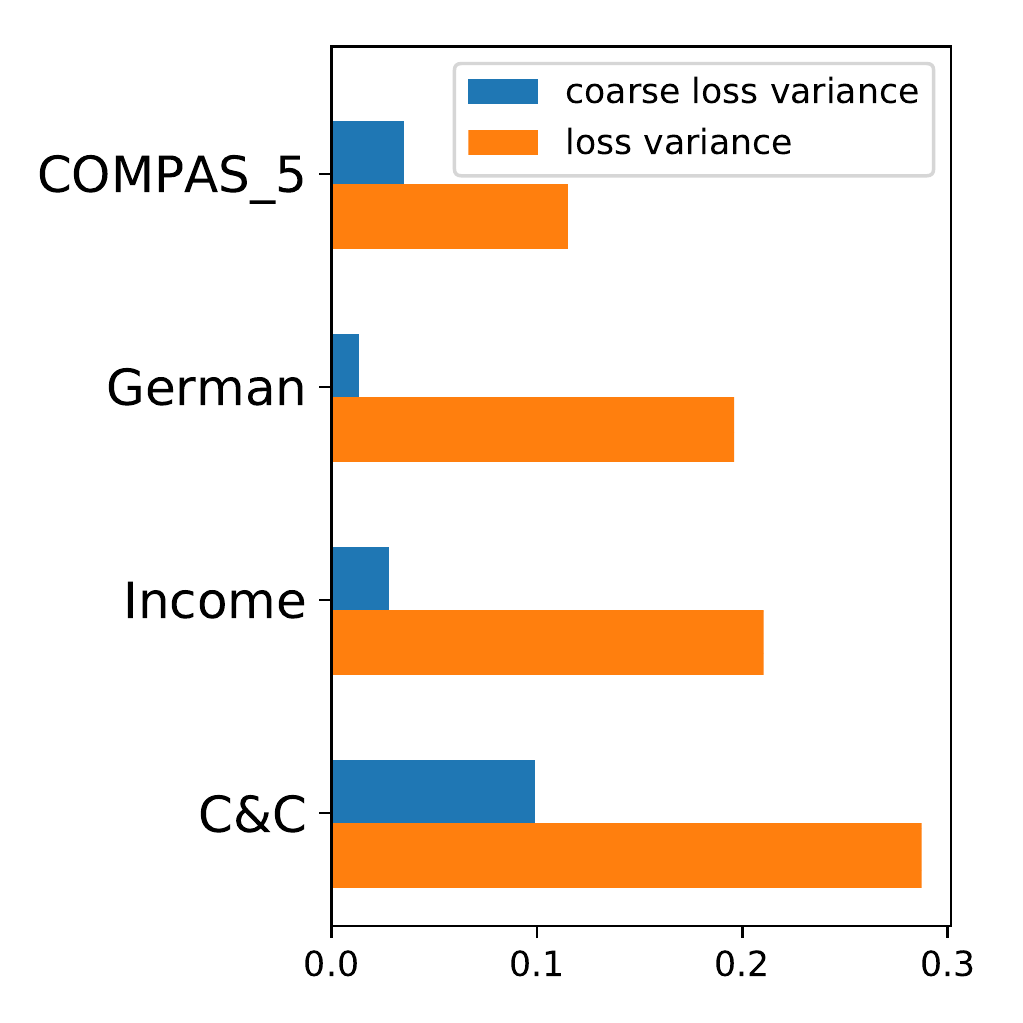}

\vspace{-0.2cm}
\tiny\hspace{0.01\textwidth}
\caption{\label{fig:LV_and_CLV} 
A logistic regression with L2-regularizer (\lr) leads to high (coarse) loss variance in all datasets.
}
\end{minipage}
\end{figure}

%

\subsection{Loss Variance Regularization}
\label{sec:loss_variance_experiments}
Recall that loss variance has
three different interpretations. 
1) It is a lower bound for maximum weighted loss discrepancy (\refprop{groupToVar});
2) It measures the average loss discrepancy between individuals and the population \refeqn{averageIndividual}; and
3) It is a regularizer to improve test error.
In this section, we study regularizing loss variance and all three aspects.
In all datasets that we consider, the effect of misclassification depends on the label.
As explained in Remark~\ref{remark:loss_discussion}, in order to not attempt to pull together the losses of individuals with different
labels, we use loss variance and coarse loss variance conditioned on the label.
Formally, we define two objectives based on loss variance and coarse loss variance as follows:
\begin{tabular}{p{0.4\textwidth}lp{0.5\textwidth}}
	{\begin{align}
	\label{eqn:\lv}
\sO_{\text{LV}} \eqdef \sO_{\text{LR}} + \lambda \widehat\E [\widehat\var [\ell\mid y]]
	\end{align}}&&
	{\begin{align}
\label{eqn:\clv}
\sO_{\text{CLV}} \eqdef \sO_{\text{LR}} + \lambda \widehat \E [\widehat\var [\widehat \E [\ell\mid A,y]\mid y]]
	\end{align}}\\
\end{tabular}
We optimize the objectives above using stochastic gradient descent.
We use a logistic regression model for prediction. In all variance computations, $\ell$ is the log loss.

\begin{figure}[t]
	\centering
	\begin{subfigure}[b]{0.25\textwidth}
		\includegraphics[width=\textwidth]{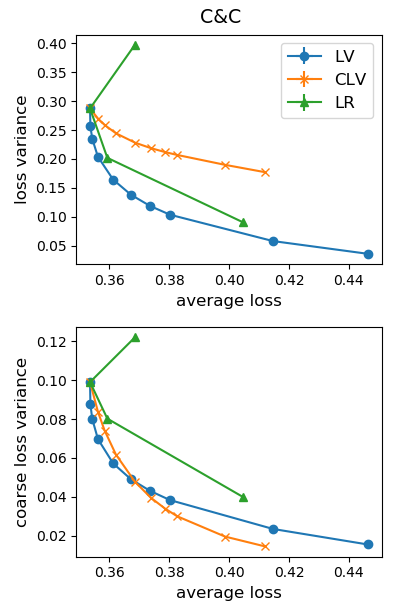}
	\end{subfigure}%
	\begin{subfigure}[b]{0.25\textwidth}
		\includegraphics[width=\textwidth]{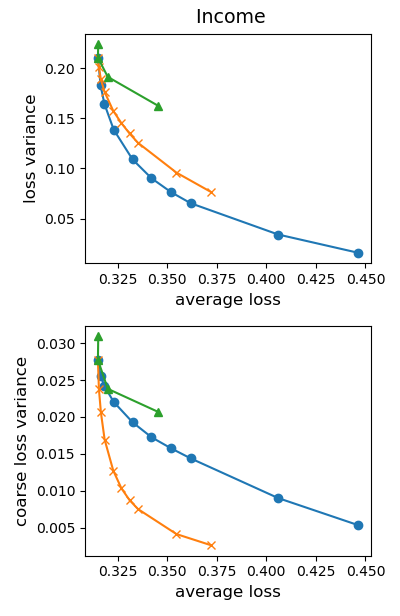}
	\end{subfigure}%
	\begin{subfigure}[b]{0.25\textwidth}
		\includegraphics[width=\textwidth]{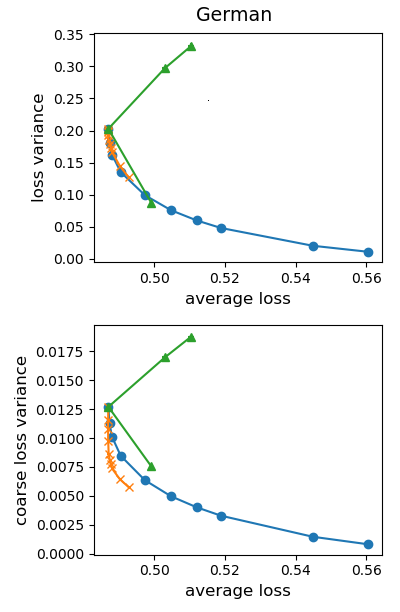}
	\end{subfigure}%
	\begin{subfigure}[b]{0.25\textwidth}
		\includegraphics[width=\textwidth]{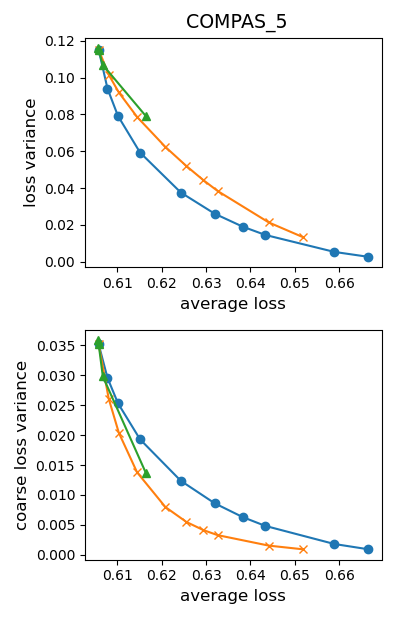}
	\end{subfigure}
	\caption{\label{fig:showing_efficiency}  
		First row: \lv{} halves the loss variance by only increasing loss 2-3\%.
		Second row: \clv{} halves the coarse loss variance by increasing the loss around 1-2\%.
	}
\end{figure}

\label{sec:results}
\begin{figure}[t]
			\centering
	\begin{minipage}{0.4\textwidth}
	\centering
	\includegraphics[width=0.8\textwidth]{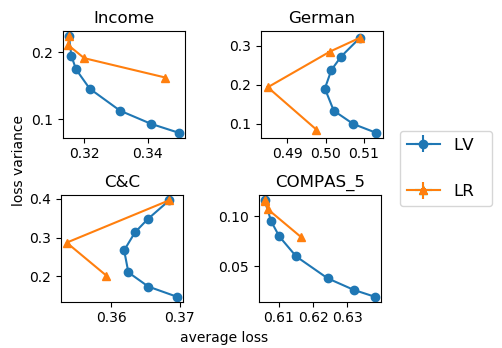}
	\captionof{figure}{\label{fig:showing_regularization} 
		\lv{} reduces loss variance and average loss simultaneously. }
	\end{minipage}\hfill%
	\begin{minipage}{0.5\textwidth}
	\centering
	\scalebox{0.85}{
		\begin{tabular}{llll}
			\toprule
			&\bf Acc. & \bf D\textsubscript{FPR}& \bf D\textsubscript{FNR} \\ \midrule
			Unconstrained &0.668&0.18&-0.30\\
			\citet{zafar2017fairness} & 0.661 & 0.03 & -0.11\\
			\citet{hardt2016} & 0.645 & -0.01 & -0.01 \\
			\clv{} ($\lambda=1.8$) &0.661 & 0.02 & -0.10\\ 
			\clv{} ($\lambda=9$) &0.656 & -0.04 & -0.04\\ \bottomrule
		\end{tabular}}
		\captionof{table}{\label{tab:group_comparison_compas_5} Comparison between different methods. D\textsubscript{FPR} (D\textsubscript{FNR}) denote the difference between False positive (False Negative) rate of white individuals and black individuals. }
	\end{minipage}
\end{figure}

\reffig{LV_and_CLV} shows that without any regularization ($\lambda = 0$, the \lr{} objective), both loss variance and coarse loss variance are large on all four datasets.
This suggests \lr{} predictions have high loss discrepancy both for groups and individuals.
Note that these two notions are incomparable across datasets---smaller loss variance does not imply smaller coarse loss variance and vice versa. 
 
Let's now evaluate the training procedures we proposed to learn a predictor with lower loss discrepancy for groups and individuals.
By varying the regularization parameter $\lambda$ in \refeqn{\lv}, we visualize the trade-off between loss variances and average loss for \lv{}.
As shown in \reffig{showing_efficiency} (first row), \lv{} halves the loss variance by increasing average loss by only 2--3\%.
Similar result is shown in the \reffig{showing_efficiency} (second row) for \clv{} \refeqn{\clv}. 
However, since
the notion of coarse loss variance allows for fluctuations in predictions across individuals with same sensitive attributes (as opposed to loss variance),
the \clv{} is able to achieve a smaller increase in average loss (1--2\%).
As a baseline, we show the trade-off curve of \lr{} by varying L2-regularizer, $\eta$.
Now we compare \lv{} and \clv{} together; in particular, we are interested in effect of \lv{} in reducing coarse loss variance.
Interestingly, in C\&C dataset \lv{} has a better trade-off curve than \clv{} in the test distribution for small value of $\lambda$.
In German dataset, unlike \clv{}, \lv{} reduced the coarse loss variance substantially in the test time.
These two observations, suggests that sometimes \lv{} might generalize to the test set better than \clv{} (as we mentioned in \refsec{fairness_CLV}).

As we discussed in \refsec{loss_variance}, loss variance has been studied as a way to improve the generalization error of a predictor.
As shown in \reffig{showing_regularization}, \lv{} reduces loss variance and loss simultaneously in German and C\&C datasets for smaller value of regularization strength on L2 \refeqn{\lv}. 
In COMPAS\_5 and Income dataset, since there are few attributes and many data points, 
neither \lv{} or \lr{} improved the loss in these two datasets.

We now shift our focus from predicting a distribution over $\sY$
to classification where the goal is to predict the label of an individual.      
We classify individual $x$ to the class $1$ if the predictor's estimate of $\pr(y=1 \mid x) > 0.5$ and $0$ otherwise.
Our approach is mainly different from previous work \citep{zafar2017fairness, hardt2016} 
as its goal is to protect all groups formed from all combinations of sensitive attributes as opposed to treating each sensitive attribute individually.
However, we compare our model and show that 
loss variance regularization reduces the maximum loss discrepancy comparable to previous work.
We pre-process COMPAS\_5 in a similar fashion to \citet{zafar2017fairness} and we compare our model to their model and \citet{hardt2016}.
We compute the difference between the false positives of blacks and whites (\dfp) and similarly the difference between false negatives of blacks and whites (\dfn).
As shown in \reftab{group_comparison_compas_5}, compared to \citet{zafar2017fairness}, our method reached lower \dfp{} and \dfn{}, even when we choose a point with same accuracy, our method still has lower \dfp{} and \dfn{}.
Compared to \citet{hardt2016}, we obtain higher accuracy but worse \dfp{} and \dfn{}. 

\section{Related work}
\label{sec:related_work}
{\bf Algorithmic fairness.}
The issue of algorithmic fairness has risen in prominence with increased prevalence of prediction \citep{barocas2016}. 
Group fairness notions which ask for some approximate parity among some predefined groups are very prevalent in fairness literature. Many group fairness notions can be viewed as instantiations of \mwld{} with different weighting functions and different loss functions (See Appendix~\ref{sec:previous_statistical_notions}).
A major thrust of our work is to guarantee fairness for all or a large number of groups,
which is shared by some recent work \citet{kearns2018gerrymandering,agarwal2018reductions,hebertjohnson2017}.
These works focus on a set of groups that can be expressed as low complexity functions of the sensitive attributes.
Depending on the complexity of the functions, estimating any fairness notion across groups in this set can be NP-hard. In contrast, we consider \emph{all} groups (appropriately weighted), which makes the estimation problem computationally tractable. 
Oblivious to the sensitive attributes, \cite{zhang2016identifying} also try to protect all groups,
using subset scan and parametric bootstrap-based methods to identify subgroups with high classification errors.
They provide some heuristic methods and only focus on finding a group with high predictive bias; whereas, we formally provide guarantees and introduce a regularizer for learning a predictor with low loss discrepancy for all groups.



{\bf Distributional robustness.}
In Distributionally Robust optimization (DRO), the broad goal is to control the
worst-case loss over distributions close to the sampling distribution $p^\star$
\citep{bental2013robust,delage2010distributionally,duchi2016,wang2016likelihood,esfahani2018data,duchi2018learning},
whereas \mwld{} measures the worst-case weighted loss discrepancy over groups, which correspond to restrictions of the support.
The two can be related as follows:
conditional value at risk (CVaR), a particular instantiation of DRO 
considers all distributions $q$ such that $q(z)/p^\star(z) \le \alpha$ for all $z$,
and we relax groups to permit fractional membership ($g$ maps to $[0, 1]$ rather than $\{0, 1\}$),
then DRO with max-norm metric is equivalent to Maximum Weighted Loss Discrepancy (MWLD) with the weighting function $w(g) = \1[\E[g] \ge \alpha]$,
which considers all groups with size at least $\alpha$.

{\bf Loss variance regularization.}
Variance regularization stems from efforts to turn better variance-based generalization bounds into algorithms.
\citet{bennett1962probability,hoeffding1963probability} show that excess risk of a hypothesis can be bounded according to its variance. 
\citet{maurer2009empirical} substitute population variance by its empirical counterpart in the excess-risk bound and introduce sample variance penalization as a regularizer.
Their analysis shows that under some settings, this regularizer can get better rates of convergence than traditional ERM. Variance regularization as an alternative to ERM also has been previously studied in  \citep{mnih2008empirical,audibert2009exploration,shivaswamy2010empirical}.
Recently, \citet{namkoong2017variance} provide a convex surrogate for sample variance penalization va distributionally robust optimization. In this work, we provide a connection between this rich literature and algorithmic fairness.



%

\section{Conclusion}
We defined and studied maximum weighted loss discrepancy (\mwld).
We gave two interpretations for \mwld{}:
1) Group fairness: it bounds the loss of any group compared to the population loss;
2) Robustness: it bounds the loss on a set of new distributions with shifted demographics, where the magnitude of the shift in a group is dictated by the weighting function of the group.
In this paper, we studied computational and statistical challenges of estimating $\mwld$ for a family of weighting functions ($w(g) = \E [g]^k$); and established a close connection between $\mwld(\ws)$ and loss variance.
This motivated loss variance regularization as a way to improve fairness and robustness. We also proposed a variant of loss variance regularization that incorporates information about sensitive attributes. 
%
We showed that we can efficiently estimate $\mwld$ for $w(g) = \E[g]^k$.  What other weighting functions does this hold for?
We relied on the key property that the sup is attained on $O(n)$ possible groups; are there other structures?

\paragraph{Reproducibility.} All code, data and experiments for this paper    are available on the Codalab platform at \url{https://worksheets.codalab.org/worksheets/0x578f01269d644524b0d4ab2a7a2a6984/}.

\paragraph{Acknowledgements.} This work was supported by Open Philanthropy Project Award.  
We would like to thank Tatsunori Hashimoto and Mona Azadkia for several helpful discussions and anonymous reviewers for useful feedback.

\bibliographystyle{plainnat}
\bibliography{refdb/all}

\begin{thebibliography}{26}
\providecommand{\natexlab}[1]{#1}
\providecommand{\url}[1]{\texttt{#1}}
\expandafter\ifx\csname urlstyle\endcsname\relax
  \providecommand{\doi}[1]{doi: #1}\else
  \providecommand{\doi}{doi: \begingroup \urlstyle{rm}\Url}\fi

\bibitem[Agarwal et~al.(2018)Agarwal, Beygelzimer, Dud{\'\i}k, Langford, and
  Wallach]{agarwal2018reductions}
A.~Agarwal, A.~Beygelzimer, M.~Dud{\'\i}k, J.~Langford, and H.~Wallach.
\newblock A reductions approach to fair classification.
\newblock \emph{arXiv preprint arXiv:1803.02453}, 2018.

\bibitem[Angwin et~al.(2016)Angwin, Larson, Mattu, and
  Kirchner]{angwin2016machine}
J.~Angwin, J.~Larson, S.~Mattu, and L.~Kirchner.
\newblock Machine bias: There’s software used across the country to predict
  future criminals. and it’s biased against blacks.
\newblock \emph{ProPublica}, 23, 2016.

\bibitem[Audibert et~al.(2009)Audibert, Munos, and
  Szepesv{'a}ri]{audibert2009exploration}
J.~Audibert, R.~Munos, and C.~Szepesv{'a}ri.
\newblock Exploration--exploitation tradeoff using variance estimates in
  multi-armed bandits.
\newblock \emph{Theoretical Computer Science}, 410\penalty0 (19):\penalty0
  1876--1902, 2009.

\bibitem[Barocas and Selbst(2016)]{barocas2016}
S.~Barocas and A.~D. Selbst.
\newblock Big data's disparate impact.
\newblock \emph{104 California Law Review}, 3:\penalty0 671--732, 2016.

\bibitem[Ben-Tal et~al.(2013)Ben-Tal, den Hertog, Waegenaere, Melenberg, and
  Rennen]{bental2013robust}
A.~Ben-Tal, D.~den Hertog, A.~D. Waegenaere, B.~Melenberg, and G.~Rennen.
\newblock Robust solutions of optimization problems affected by uncertain
  probabilities.
\newblock \emph{Management Science}, 59:\penalty0 341--357, 2013.

\bibitem[Bennett(1962)]{bennett1962probability}
G.~Bennett.
\newblock Probability inequalities for the sum of independent random variables.
\newblock \emph{Journal of the American Statistical Association (JASA)},
  57\penalty0 (297):\penalty0 33--45, 1962.

\bibitem[Berk(2012)]{berk2012criminal}
R.~Berk.
\newblock \emph{Criminal justice forecasts of risk: A machine learning
  approach}.
\newblock Springer Science \& Business Media, 2012.

\bibitem[Bolukbasi et~al.(2016)Bolukbasi, Chang, Zou, Saligrama, and
  Kalai]{bolukbasi2016man}
T.~Bolukbasi, K.~Chang, J.~Y. Zou, V.~Saligrama, and A.~T. Kalai.
\newblock Man is to computer programmer as woman is to homemaker? debiasing
  word embeddings.
\newblock In \emph{Advances in Neural Information Processing Systems
  (NeurIPS)}, pages 4349--4357, 2016.

\bibitem[Delage and Ye(2010)]{delage2010distributionally}
E.~Delage and Y.~Ye.
\newblock Distributionally robust optimization under moment uncertainty with
  application to data-driven problems.
\newblock \emph{Operations research}, 58\penalty0 (3):\penalty0 595--612, 2010.

\bibitem[Duchi and Namkoong(2018)]{duchi2018learning}
J.~Duchi and H.~Namkoong.
\newblock Learning models with uniform performance via distributionally robust
  optimization.
\newblock \emph{arXiv preprint arXiv:1810.08750}, 2018.

\bibitem[Duchi et~al.(2016)Duchi, Glynn, and Namkoong]{duchi2016}
J.~Duchi, P.~Glynn, and H.~Namkoong.
\newblock Statistics of robust optimization: A generalized empirical likelihood
  approach.
\newblock \emph{arXiv}, 2016.

\bibitem[Esfahani and Kuhn(2018)]{esfahani2018data}
P.~M. Esfahani and D.~Kuhn.
\newblock Data-driven distributionally robust optimization using the
  wasserstein metric: Performance guarantees and tractable reformulations.
\newblock \emph{Mathematical Programming}, 171\penalty0 (1):\penalty0 115--166,
  2018.

\bibitem[Hardt et~al.(2016)Hardt, Price, and Srebo]{hardt2016}
M.~Hardt, E.~Price, and N.~Srebo.
\newblock Equality of opportunity in supervised learning.
\newblock In \emph{Advances in Neural Information Processing Systems
  (NeurIPS)}, pages 3315--3323, 2016.

\bibitem[Hashimoto et~al.(2018)Hashimoto, Srivastava, Namkoong, and
  Liang]{hashimoto2018repeated}
T.~B. Hashimoto, M.~Srivastava, H.~Namkoong, and P.~Liang.
\newblock Fairness without demographics in repeated loss minimization.
\newblock In \emph{International Conference on Machine Learning (ICML)}, 2018.

\bibitem[H{\'e}bert-Johnson et~al.(2017)H{\'e}bert-Johnson, Kim, Reingold, and
  Rothblum]{hebertjohnson2017}
{\'U}rsula H{\'e}bert-Johnson, M.~P. Kim, O.~Reingold, and G.~N. Rothblum.
\newblock Calibration for the (computationally-identifiable) masses.
\newblock \emph{arXiv preprint arXiv:1711.08513}, 2017.

\bibitem[Hoeffding(1963)]{hoeffding1963probability}
W.~Hoeffding.
\newblock Probability inequalities for sums of bounded random variables.
\newblock \emph{Journal of the American Statistical Association}, 58\penalty0
  (301):\penalty0 13--30, 1963.

\bibitem[Kamiran and Calders(2012)]{kamiran2012data}
F.~Kamiran and T.~Calders.
\newblock Data preprocessing techniques for classification without
  discrimination.
\newblock \emph{Knowledge and Information Systems}, 33\penalty0 (1):\penalty0
  1--33, 2012.

\bibitem[Kearns et~al.(2018)Kearns, Neel, Roth, and
  Wu]{kearns2018gerrymandering}
M.~Kearns, S.~Neel, A.~Roth, and Z.~S. Wu.
\newblock Preventing fairness gerrymandering: Auditing and learning for
  subgroup fairness.
\newblock \emph{arXiv preprint arXiv:1711.05144}, 2018.

\bibitem[Massart(1990)]{massart1990tight}
P.~Massart.
\newblock The tight constant in the dvoretzky-kiefer-wolfowitz inequality.
\newblock \emph{The annals of Probability}, pages 1269--1283, 1990.

\bibitem[Maurer and Pontil(2009)]{maurer2009empirical}
A.~Maurer and M.~Pontil.
\newblock Empirical bernstein bounds and sample variance penalization.
\newblock \emph{arXiv preprint arXiv:0907.3740}, 2009.

\bibitem[Mnih et~al.(2008)Mnih, Szepesv{'{a}}ri, and
  Audibert]{mnih2008empirical}
V.~Mnih, C.~Szepesv{'{a}}ri, and J.~Audibert.
\newblock Empirical berstein stopping.
\newblock In \emph{International Conference on Machine Learning (ICML)}, 2008.

\bibitem[Namkoong and Duchi(2017)]{namkoong2017variance}
H.~Namkoong and J.~Duchi.
\newblock Variance regularization with convex objectives.
\newblock In \emph{Advances in Neural Information Processing Systems
  (NeurIPS)}, 2017.

\bibitem[Shivaswamy and Jebara(2010)]{shivaswamy2010empirical}
P.~Shivaswamy and T.~Jebara.
\newblock Empirical {B}ernstein boosting.
\newblock In \emph{Artificial Intelligence and Statistics (AISTATS)}, pages
  733--740, 2010.

\bibitem[Wang et~al.(2016)Wang, Glynn, and Ye]{wang2016likelihood}
Z.~Wang, P.~W. Glynn, and Y.~Ye.
\newblock Likelihood robust optimization for data-driven problems.
\newblock \emph{Computational Management Science}, 13\penalty0 (2):\penalty0
  241--261, 2016.

\bibitem[Zafar et~al.(2017)Zafar, Valera, Rodriguez, and
  Gummadi]{zafar2017fairness}
M.~B. Zafar, I.~Valera, M.~G. Rodriguez, and K.~P. Gummadi.
\newblock Fairness beyond disparate treatment \& disparate impact: Learning
  classification without disparate mistreatment.
\newblock In \emph{World Wide Web (WWW)}, pages 1171--1180, 2017.

\bibitem[Zhang and Neill(2016)]{zhang2016identifying}
Z.~Zhang and D.~B. Neill.
\newblock Identifying significant predictive bias in classifiers.
\newblock \emph{arXiv preprint arXiv:1611.08292}, 2016.

\end{thebibliography}

\appendix
\onecolumn
\renewcommand{\thesection}{\Alph{section}}
\section{Previous statistical notions}
\label{sec:previous_statistical_notions}
Statistical notions of fairness can be viewed as instantiations of \mwld{} (\refdef{mwld}) with different weighting functions and appropriate loss functions.
We categorize existing notions into three rough categories and flesh out the associated weighting functions. 
\paragraph{Group fairness:}
Early work of statistical fairness \cite{kamiran2012data,hardt2016,zafar2017fairness} only control discrepancy of loss for a small number of groups defined on sensitive attributes (e.g., race and gender).
This corresponds to a zero-one weighting function where weights of the fixed sensitive groups are $1$, and weights of all other groups are $0$.
\paragraph{Subgroup fairness:}
\citet{kearns2018gerrymandering} argue that group fairness is prone to the ``fairness gerrymandering'' problem whereby groups corresponding to combinations of sensitive attributes can have high loss even if groups corresponding to sensitive attributes individually are protected.\footnote{
	They show an illustrative example in which a predictor where black men and white women suffer very small loss but black women and white men suffer very high loss. Such a predictor has similar losses on the groups of blacks and whites (corresponding to the sensitive attribute race) and men and women (corresponding to gender); however, has a high loss on the ``sub-groups'' defined on the combination of the attributes.}
To mitigate this issue, they consider exponentially many subgroups ($G_\text{subgroup}$) defined by a structured class of functions over the sensitive attributes and control loss of each group in $G_\text{subgroup}$ weighted by its size.
Formally, their definition corresponds to Definition~\ref{def:mwld} with the weighting function $w(g) = \1[g \in G_\text{subgroup}]\E [g]$. 
\paragraph{Large-group fairness:} \citet{hashimoto2018repeated} also consider the setting where sensitive attributes are unknown (which is also our main focus), 
They aim to control the losses of groups whose size is greater than some predefined value $\alpha$ (oblivious to sensitive attributes). 
In terms of Definition~\ref{def:mwld}, this corresponds to the weighting function $w(g)=\1[\E [g] \ge \alpha]$.

\section{Missing Proofs}
\label{sec:appendix}

\begin{proposition}
	\label{prop:robustness}
	Fix any loss function $0\le \ell \le 1$ and weighting function $w(g)$ such that $w(g) \ge \E [g]$. For $q(\cdot) \eqdef w(g) p^\star(\cdot \mid g = 1) + (1 - w(g)) p^\star(\cdot \mid g = 0)$, the following holds:
	
	\begin{align}
	\E_{z \sim q}[\ell] \le \E_{z \sim p^\star} [\ell] + \mwld(w).
	\end{align}
\end{proposition}
\begin{proof}
	We prove the following:
	\begin{align}
	\label{eqn:robustness}
		\E_{z \sim q}[\ell] = w(g) \E_{z \sim p^\star} [\ell \mid g=1] + (1-w(g)) \E_{z \sim p^\star} [\ell \mid g=0] \le \E_{z \sim p^\star} [\ell] + \mwld(w)
	\end{align}
	 
	In the rest of the proof, all expectations are with respect to $p^\star$.
	 
	If $\E [\ell \mid g=1] \le \E [\ell]  \le \E [\ell \mid g=0]$ then it is obvious that increasing weight of group $g$ will decrease the overall loss:
	\begin{align}
	\nonumber w(g) \E [\ell \mid g=1] + (1-w(g)) \E [\ell \mid g=0] =& 
	\E [\ell \mid g = 0] - w(g) \p {\E [\ell \mid g=0] - \E [\ell \mid g=1]}\\
	\nonumber \le &\E [\ell \mid g = 0] - \E [g] \p {\E [\ell \mid g=0] - \E [\ell \mid g=1]} \\
		\nonumber=&\E [g]\E [\ell \mid g = 1] + (1-\E[g])  \E [\ell \mid g=0]\\
	\le &\E [\ell] \le \E [\ell ] + \mwld(w)
	\end{align}
	
	If $\E [\ell \mid g=0] \le \E [\ell] \le \E [\ell \mid g=1]$, as shown in \refeqn{guarantee} by definition of maximum weighted loss discrepancy we have $\E [\ell \mid g=1] \le \E [\ell] + \frac{\mwld (w)}{w (g)}$. We can bound the RHS of \refeqn{robustness} as follows:
	\begin{align}
	\nonumber w(g) \E [\ell \mid g=1] + (1-w(g)) \E [\ell \mid g=0] \le & w (g) \E [\ell] + \mwld (w) +\\
	\nonumber	&(1-w(g)) \E [\ell \mid g=0]\\
	\le &\E [\ell] + \mwld(w)
	\end{align}
\end{proof}

\impossibility*

\begin{proof}
	Consider distribution $p_1$ such that $\mwld(\weq, \ell, h) < \half$ under this distribution:
	\begin{align*}
	\mwld_1 &\eqdef \max \limits_{g \in \sG, \E[g] > 0} \Big | \E_{p_1}[\ell(h, z) \mid g(z) = 1] -  \E_{p_1}[\ell(h, z)] \Big| 
	 < \half ~~ \text{(By assumption)}.
	\end{align*}  
	We now construct a new distribution such that $\mwld(\weq, \ell, h) \ge \half$ for this distribution.
	Let $z_0$ and $z_1$ be two points with loss $0$ and $1$ respectively (i.e., $\ell (h,z_0) =0$, $\ell (h,z_1) = 1$).        
	We construct a new distribution $p_2$ as follows: for $0<\eta <1$ 
	with probability $(1 - \eta)$, $z \sim p_1$ and $z = z_\text{1}$ and $z = z_\text{0}$ with probability $\frac{\eta}{2}$ each.
	
	The maximum weighted loss discrepancy for this distribution $\mwld_2$ is defined analogous to $\mwld_1$ above.
	By the existence of groups $g_\text{1}, g_\text{0}$ corresponding to the singletons $z_\text{1}$ and $z_\text{0}$ with $\E[g_\text{1}] = \E[g_\text{0}] = \frac{\eta}{2} > 0$,

	We now assume an estimator exists and will show a contradiction.
	Let $\gamma_1$ and $\gamma_2$ denote two random variables corresponding to the estimates of $\mwld_1$ and $\mwld_2$ respectively; 
	set $\epsilon = \frac{1}{4} - \frac{\mwld_1}{2}$, we have:
	
	\begin{align}
	\pr \p{|\gamma_2 - \mwld_2 | \ge \epsilon} &\ge \pr \p{\gamma_2 - \mwld_2 \le -\epsilon}\\
	&= \pr \p{\gamma_2 \le \mwld_2 - \epsilon}\\
	&\ge \pr \p{\gamma_2 \le \half - \epsilon }\\
	&= \pr \p {\gamma_2 \le \frac{1}{4} + \frac{\mwld_1}{2}}\\
	&= \pr \p {\gamma_2 \le \mwld_1 + \p {\frac{1}{4} - \frac{\mwld_1}{2}}}\\
	&= \pr \p{\gamma_2 \le \mwld_1+ \epsilon}\\
	&\ge \pr \p{|\gamma_2 - \mwld_1| \le \epsilon}\\
	&\ge \p{1-\eta}^n\pr \p{|\gamma_1 - \mwld_1| \le \epsilon}\\
	&\ge \p{1-\eta}^n\p{1-\delta} \end{align}
	
Now for some $\delta < \hat \delta < \half$
Set $\eta = 1-\sqrt[n]{\frac{\hat \delta}{1-\delta}}$ then we have:
	\begin{align}
	\pr (|\gamma_2 - \mwld_2 | \ge \epsilon) 	&=\p {1-\p {1-\sqrt[n]{\frac{\hat \delta}{1-\delta}}}}^n (1-\delta) \\
	&= \hat \delta > \delta,
	\end{align}
	which contradicts with the assumption that $\pr \left[ \pab{\gamma_2 - \mwld_2} > \epsilon \right] \le \delta$.
\end{proof}

\propGroupToVar*

\begin{proof}
We first prove $\sqrt{\var [\ell]} \le \mwld(\ws)\sqrt{2 - 4\ln\p{\mwld(\ws)}}$.

We assume a more general case $0\le \ell \le L$.
We subtract the average loss, $\mu$, from the loss of each point to center the losses and make the average loss $0$ (without changing the MWLD).
First note that if $\mwld(\ws)=0$, then $\var [\ell]=0$.
If $\mwld(\ws) >0$, let $F$ be the cumulative density function (CDF) of $\ell$.
\begin{align}
\var [\ell] = \int_{-L}^{0} u^2 dF(u) + \int_0^L u^2 dF(u)
\end{align}

We now derive an upper bound for $ \int_0^L u^2dF(u)$; we can compute an upper bound for $ \int_{-L}^{0} u^2 dF(u)$ in a similar way.
For an arbitrary $u>0$, consider the group of individuals with loss $\geq u$. 
By bounding the weighted loss of this group using MWLD, we obtain a bound on $\pr(\ell \ge u)$. 
For brevity, let $\gamma = \mwld(\ws)$.
\begin{align}
\label{eqn:upperbound}
\pr(\ell \ge u) u^2 \le \pr(\ell \ge u)\E [\ell \mid \ell \ge u]^2 \le \gamma^2 
\implies \pr(\ell \ge u)\le \frac{\gamma^2}{u^2}
\end{align}

Using integration by parts, we can express $\E [\ell^2]$ using $\pr (\ell \ge u)$. We then bound $\E [\ell^2]$ using the above upper bound \refeqn{upperbound}. 
\begin{align}
\int_{0}^{L} u^2dF(u)
=&   \L - \pr (\ell \ge u) u^2 \Big\rvert_{0}^L - 2\int_{0}^{L} -\pr(\ell \ge u)udu\\  
=&2\int_{0}^{L} \pr(\ell \ge u)udu\\
=&   2\int_{0}^{\gamma} u\pr(\ell \ge u)du+
2\int_{\gamma}^{L} u\pr(\ell \ge u)du\\
\le& 2\int_{0}^{\gamma} udu+
2\int_{\gamma}^{L} u\frac{\gamma^2}{u^2}du\\
\le& \gamma^2 + 2\gamma^2\ln\p{\frac{L}{\gamma}}
\end{align}

By computing a similar bound for $ \int_{-L}^{0} u^2 dF(u)du$, the following holds:

\begin{align}
\var [\ell] \le 2\gamma^2\p {1+ 2\ln\p{\frac{L}{\gamma}}}.
\end{align}
Setting $L=1$ we have:
\begin{align}
\var [\ell] \le \gamma^2\p {2 - 4\ln\p{\gamma}}
\end{align}

Now we prove $\mwld(\ws) \le \sqrt {\var [\ell]}$.

	We prove for any $g$, $\E [g] \p {\E [\ell \mid g] - \E [\ell]}^2 \le \var[\ell]$, using law of total variance.  
	Note that $\E [g(z)] = \pr(g(z)=1)$.
	\begin{align}
	&\var[\E[\ell\mid g]] + \E [\var [\ell\mid g]] = \var[\ell] \\
	&\var [\E [\ell \mid g]] = \E [g]\E [1-g](\E[\ell|g=1] - \E[\ell|g=0])^2 \le \var[\ell]\\
	&\E [g]\E [1-g]\left(\frac{\E [g]\E[\ell\mid g=1]}{\E [g]}-\frac{\E [1-g]\E[\ell\mid g=0]}{\E [1-g]}\right)^2\le \var[\ell]
	\end{align}
	Now we that use the fact $\E [g]\E[\ell \mid g=1] + \E [1-g] \E[\ell \mid g=0] = \E[\ell]$.
	\begin{align}
	&\E [g]\E [1-g]\left(\frac{\E [g]\E[\ell\mid g=1]}{\E [g]}-\frac{\E [\ell]- \E [g]\E [\ell \mid g=1]}{\E [1-g]}\right)^2\le \var[\ell]\\
	&\frac{\Big(\E [g]\E [\ell\mid g=1]\left(\E [1-g] + \E [g]\right) - \E [g]\E [\ell]\Big)^2}{\E [g]\E [1-g]}\le \var[\ell]\\
	&\label{eqn:stronger_bound} \frac{\E [g]\left(\E[\ell\mid g=1] - \E[\ell]\right)^2}{\E [1-g]}\le \var[\ell] \\
	& \E [g]\left(\E[\ell\mid g=1] - \E[\ell]\right)^2 \le \var[\ell]\\
	&\mwld(\ws) \le \sqrt{\var [\ell]}
	\end{align}

\end{proof}

\begin{corollary}
\label{cor:protectedGroup}
For any measurable loss function $0 \le \ell \le 1$. The following holds,
\begin{align}
\label{eqn:way1p}
\mwld(\wsa) &\leq \sqrt{\var[\E [\ell \mid A]]} \le \mwld(\wsa)\sqrt{2-4\ln(\mwld(\wsa))}
\end{align}
\end{corollary}

\begin{proof}
Let $\alpha$ be the random variable indicating the different values of sensitive attributes ($A$), with the following distribution: $\pr(\alpha) = \pr\left(A(z) = \alpha\right)$; and loss on point $\alpha$ is defined as the expected loss of all point with sensitive attribute $\alpha$, formally: $\ell(h,\alpha) = \E(\ell(h,z) \mid A(z)=\alpha)$. 
As $g$ is only defined on $A$ we can now use $\alpha$ in the \refprop{groupToVar} to prove the corollary.
\end{proof}

\section{Generalization bounds}
\label{sec:generalization_bounds}

\begin{theorem}[\citet{maurer2009empirical}]
	\label{thm:variance_generalizes}
	Let $n\ge2 $, and $z_1, z_2, \hdots z_n$ denote the training data. Let $\ell$ be the losses of individuals in the training data, with values in
	$[0,1]$. 
	Let $\widehat\var$ denote the empirical variance, then for any $\delta > 0$, we have
	\begin{align}
	\pr \Big[ \left|\sqrt{\var[\ell]} - \sqrt{\widehat\var_n[\ell]}\right| > \sqrt { \frac{2 \ln \frac{2}{\delta}} {n-1}} \Big] \leq \delta.
	\end{align}
\end{theorem}

\begin{theorem}[Hoeffding's inequality]
	\label{thm:hoeffding}
Let $\ell_1,\dots,\ell_n$, be $n$ i.i.d. random variables with values in $[0,1]$ and let $\delta >0$. Then with probability at least $1-\delta$ we have:
\begin{align}
\left|\E [\ell] - \widehat\E [\ell]\right| \le \sqrt{\frac{\ln{2/\delta}}{2n}}
\end{align}
\end{theorem}

As a caveat, empirical coarse loss variance converges to its population counterpart slower than loss variance.
Let $T$ be the number of different settings of the sensitive attributes ($T=\log (|\sG_A|)$).
As an example, if we have two sensitive attributes with each having $20$ different values, then $T=400$ and $\lvert \sG_A \rvert = 2^{400}-1$.
Empirical coarse loss variance converges to the population coarse variance as $\sqrt{T/n}$, while empirical loss variance convergences to its population counterpart as $\sqrt{1/n}$.
In the following theorem, we present the formal proof for this bound.


\begin{theorem}
\label{thm:coarse_variance_generalizes}
Let $n$ denote the number of training data points;
and $0 \le \ell \le $ be the loss function. 
Let $A$ be the set of all possible settings of sensitive features, and $T$ be the number of possible settings.
Then for any $\delta > 0$, and $n > 2$, with probability of at least $1-(T+3)\delta$, the following holds:
\begin{align}
\left\lvert\var [ \E [\ell \mid A]] - \widehat\var [\widehat \E [\ell \mid A]]\right\rvert \le \sqrt{\frac{2\ln(\frac{2}{\delta})}{n-1}} + 
\sqrt{\frac{(2T+8)\ln(\frac{2}{\delta})}{n}}
\end{align}
\end{theorem}
\begin{proof}
    Throughout the proof, we use this fact that if two random variables and their difference are bounded, then the difference of their square is bounded as well.
	Formally, for any $0 \le a,\widehat a \le M$ we have:
	\begin{align}
	\label{eqn:square}
	|a^2 - \widehat a^2| = (a+\widehat a) |a - \widehat a|\le 2M|a-\widehat a|
	\end{align}

	For simplifying the notations, let $Y$ and $\widehat{Y}$ denote $\E [\ell \mid A]$ and let $\widehat\E [\ell \mid A]$ respectively.
	
\begin{align}
\left\lvert\var [ Y] - \widehat\var [\widehat Y]\right\rvert =
&\left\lvert\var [ Y] - \widehat\var [ Y] + \widehat\var [Y] - \widehat\var [\widehat Y]\right\rvert \\
\le &\left\lvert\var [ Y] - \widehat\var [ Y] \right\rvert + \left\lvert \widehat\var [ Y] - \widehat\var [\widehat Y]\right\rvert\\
=&\left\lvert \var \left[ Y\right] - \widehat\var [  Y] \right\rvert + 
\left\lvert\widehat\E [Y^2] - \widehat\E [Y]^2 - \widehat\E [\widehat Y^2] + \widehat\E [\widehat Y]^2\right\rvert \\
\le &
\underbrace{\left\lvert \var \left[ Y\right] - \widehat\var [  Y] \right\rvert}_{(i)} + \underbrace{\left\lvert\widehat\E [Y^2] - \widehat\E [\widehat Y^2]\right\rvert}_{(ii)} + 
\underbrace{\left\lvert\widehat\E [\widehat Y]^2  - \widehat\E [Y]^2 \right\rvert}_{(iii)}
\end{align}

Using \refthm{variance_generalizes} and \refeqn{square} while considering $\sqrt{\var[Y]} \le 0.5$, with probability at least $1-\delta$, (i) is less than $\sqrt{\frac{2\ln(\frac{2}{\delta})}{n-1}}$.

We now compute an upper bound for (iii). First note that: 
$\E [\widehat\E [Y]] = \E [Y] = \E [\E [\ell \mid A]] = \E [\ell]$ and similarly,
$\E [\widehat\E [\widehat Y]] = \E [\widehat\E [\widehat\E [ \ell \mid A]]] = \E [\widehat\E [\ell]] = \E [\ell]$.

\begin{align}
\left\lvert\widehat\E [\widehat Y]^2  - \widehat\E [Y]^2 \right\rvert 
=& \left\lvert\widehat\E [\widehat Y]^2  - \E [\ell]^2 + \E [\ell]^2 - \widehat\E [Y]^2 \right\rvert \\
\le & 
\left\lvert\widehat\E [\ell]^2 - \E [\ell]^2 \right\rvert +  \left\lvert\E [\widehat\E [ Y]]^2 - \widehat\E [Y]^2 \right\rvert\\
\le & 4 \sqrt{\frac{\ln \p {2/\delta}}{2n}} \label{eqn:hoefdingUse}\\
= & \sqrt{\frac{8\ln \p {2/\delta}}{n}}
\end{align}

We derived \refeqn{hoefdingUse}, using \refthm{hoeffding} and \refeqn{square} considering $0 \le \ell \le 1$.

Finally, we compute an upper bound for (ii). Let $0 \le n_i \le n$ denote the number of time that sensitive attributes setting of $a_i$ appeared in the training data.
\begin{align}
\left\lvert\widehat\E [Y^2] - \widehat\E [\widehat Y^2]\right\rvert =&\frac{1}{n}\left |\sum_{i=1}^{T} n_i \p {\E [\ell \mid A=a_i]^2 - \widehat\E [\ell \mid A=a_i]^2}\right|\\
\le&\frac{1}{n}\sum_{i=1}^{T} n_i \left|\E [\ell \mid A=a_i]^2 - \widehat\E [\ell \mid A=a_i]^2\right|\\
\le& \frac{1}{n}\sum_{i=1}^{T} 2n_i \sqrt{\frac{\ln\p{2/\delta}}{2n_i}}\\
=&\sqrt{2\ln\p{2/\delta}} \sum_{i=1}^{T}  \sqrt{\frac{n_i}{n^2}}\label{eqn:hoefdingUse2}\\
\le& \sqrt{\frac{2T\ln\p{2/\delta}}{n}}
\end{align}

We derived \refeqn{hoefdingUse2}, using \refthm{hoeffding}, and \refeqn{square} considering $0 \le \ell \le 1$.

\end{proof}

\section{Uniform convergence}
The following intermediate lemma is used in proof of \refthm{auditor}.
\sidecaptionvpos{figure}{t}
\begin{SCfigure}
	\includegraphics[width=0.25\textwidth]{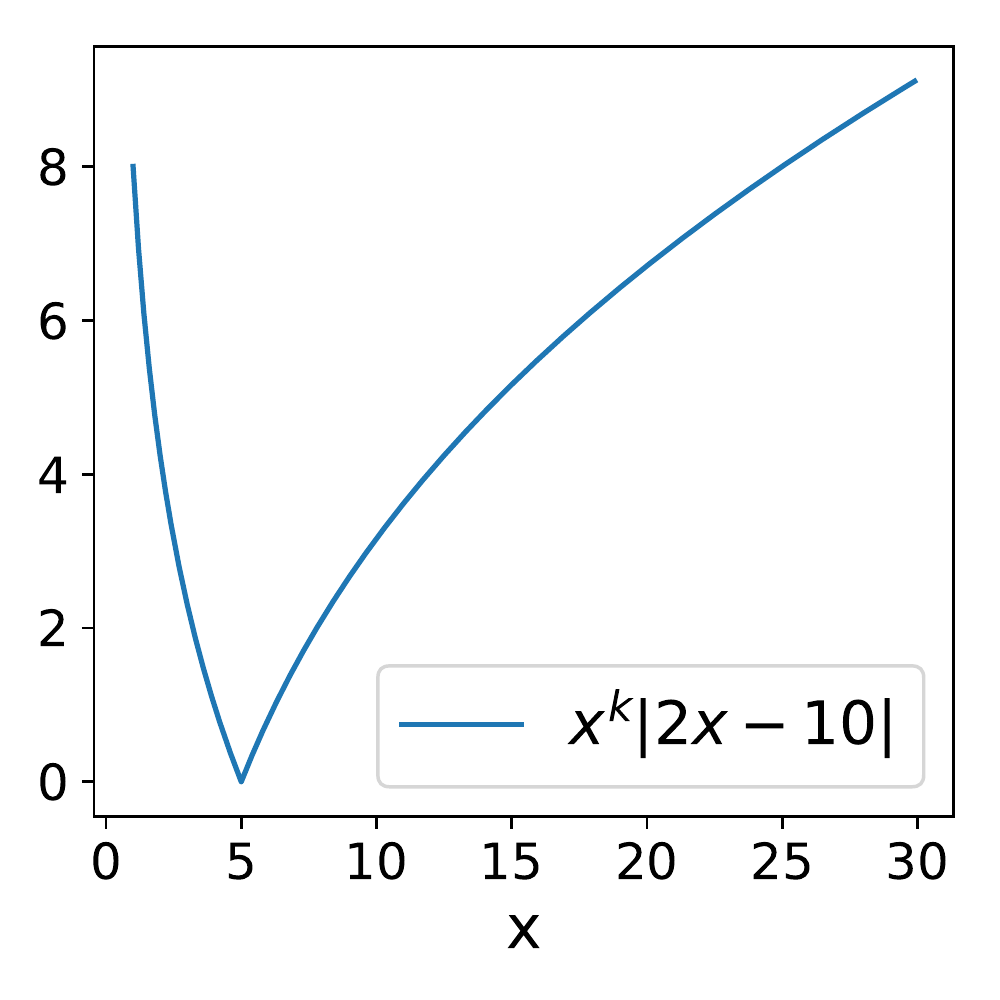}%
	\includegraphics[width=0.25\textwidth]{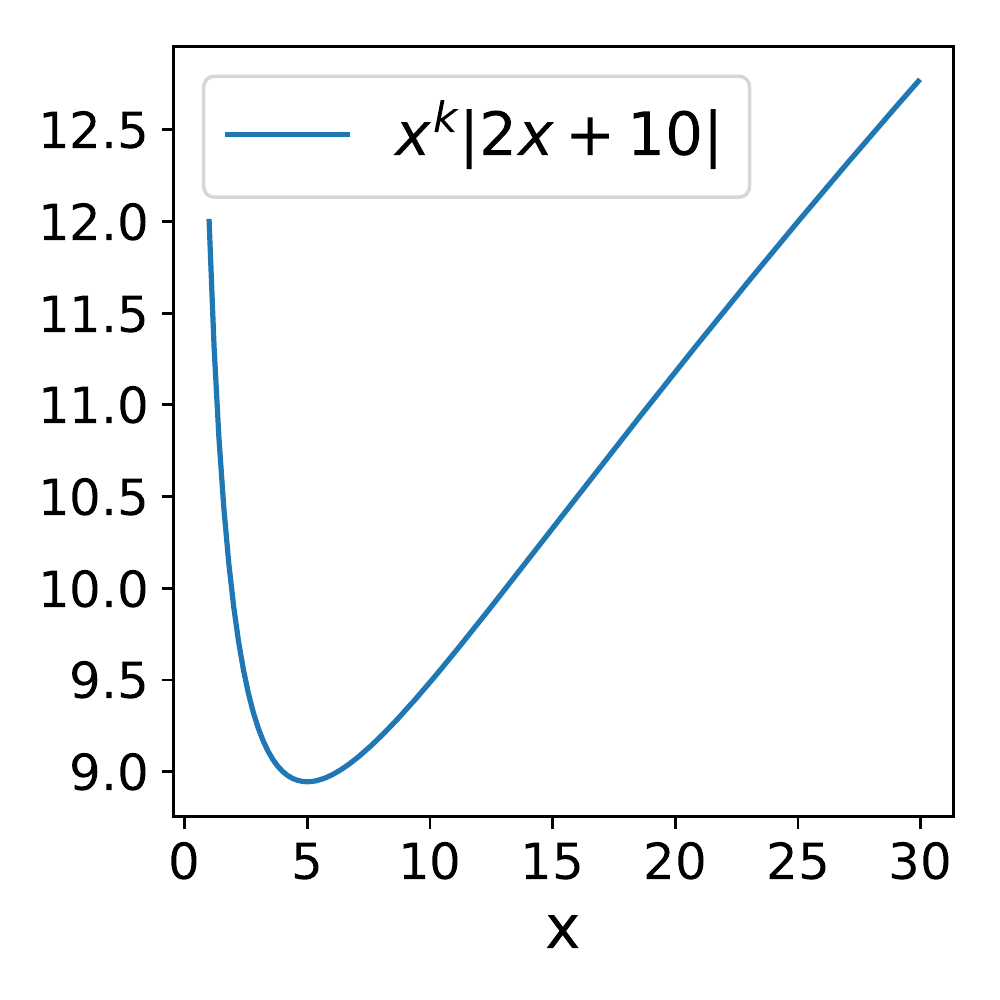}
	\caption{\protect\rule{0ex}{5ex}
		\label{fig:lemma1_figure} Two different graphs for $f(x)$ as explained in \reflem{computing_U_efficiently} for different values of $a$ and $b$.}
\end{SCfigure}
\begin{lemma}
	\label{lem:computing_U_efficiently}
Given $n$ data points $z_1,\dots,z_n$, let $\widehat\mwld(\wk)$ be as defined in \refeqn{empirical_mwld}.
There exists a real number $u$ such that either the group containing data points with loss less than $u$ or the group containing data points with loss greater than $u$ has the maximum empirical weighted loss discrepancy.
\end{lemma}

\begin{proof}

WLOG, assume $\widehat \E[\ell]=0$ and data points are sorted according to their loss in ascending order, i.e., $\ell_1\le \dots\le\ell_n$.
Let $\widehat D(g)$ denote the weighted loss discrepancy for group $g$.

\begin{align}
\widehat D(g) \eqdef \widehat\E [g]^k \pab {\widehat\E [\ell \mid g=1] - \widehat\E [\ell]}
\end{align}
Let $g^\star$ be any group with maximum weighted loss discrepancy.
\begin{align}
g^\star = \text{argmax}_{g\in{\sG}} \widehat D(g)
\end{align} 

WLOG, assume $\widehat \E [\ell \mid g^\star=1] \le \widehat \E [\ell]$ (we can multiply every loss by $-1$ if $\widehat \E [\ell \mid g^\star=1] > \widehat \E [\ell]$).
We now prove there exits an integer $1\le r\le n$ such that $g^\star = \{z_1, \dots, z_r\}$ and $\ell_r \neq \ell_{r+1}$.
Assume this is not the case.
Let $p$ be the smallest number such that $z_p \notin g^\star$ but for a $q > p$  we have $z_q \in g^\star$. 
If $\ell_p < \ell_q$  we can replace $z_p$ by $z_q$ in $g^\star$, this replacement does not change size of $g^\star$ but it increases $\widehat \E [\ell]- \widehat \E [\ell \mid g^\star=1]$ which is a contradiction to the assumption that $g^\star$ is a group with maximum weighted loss discrepancy.

We now prove $\ell_p$ cannot be equal to $\ell_q$.
For contradiction assume they are equal and $\ell_q=\ell_p = a$.
We prove if $\widehat D(g^\star/\{z_q\}) \le \widehat D(g^\star)$ then $\widehat D(g^\star) \le \widehat D(g^\star \cup \{z_p\})$ which is a contradiction with the assumption that $g^\star$ is the group with maximum weighted loss discrepancy.
Assume there are $t$ points inside $g^\star/\{z_q\}$ with sum $s$ then we have:
\begin{align}
&\widehat D(g^\star/\{z_q\}) = \p{\frac{t}{n}}^k\pab{\frac{s}{t}}\\
&\widehat D(g^\star) = \p{\frac{t+1}{n}}^k\pab{\frac{s+a}{t+1}}\\
&\widehat D(g^\star \cup\{z_p\}) = \p{\frac{t+2}{n}}^k\pab{\frac{s+2a}{t+2}}\\
\end{align}

Let $b=s-at$ and define a function $f(x) \eqdef n^{-k}x^{k-1}|ax+b|$. 
Therefore $f(t) = \frac{t^{k-1}}{n^k}|s| = \p {\frac{t}{n}}^k \pab {\frac{s}{t}} =  \widehat D(g^\star/\{z_q\})$, and  $f(t+1) = \frac{(t+1)^{k-1}}{n^k}\pab{s+a}$ is the weighted loss discrepancy for $g^\star$, and finally $f(t+2) = \frac{(t+2)^{k-1}}{n^k} |s+2a|$ is the weighted loss discrepancy for $g^\star \cup \{z_p\}$.

We want to prove $f(x) \le f(x+1) \implies f(x+1) \le f(x+2)$.
Let's look at the derivative of $f$, for $x>0$ we have:
\begin{align}
f' = \frac{ax^k+bx^{k-1}}{|ax^k+bx^{k-1}|}\p { akx^{k-1} + b(k-1)x^{k-2}} = \frac{ax+b}{|ax+b|}x^{k-2} \p {akx + b(k-1)}
\end{align}
Sign of $f'$ changes at $-\frac{b}{a}$ and  $-\frac{b(k-1)}{ak}$; since we assumed $k \le 1$ only one of these two numbers are positive.
Assume $-\frac{b}{a}$ is positive then sign $f'$ is negative in $x \in (0, -\frac{b}{a})$ and become positive in $x \in (-\frac{b}{a}, \infty)$. 
Therefore, $f(x) \le f(x+1)$ implies $f(x+1) \le f(x+2)$.
\reffig{lemma1_figure} shows graph of $f(x)$ for different values of $a$ and $b$ and for $k=0.5$.

This result enables us to compute $\widehat \mwld(\wk)$ efficiently. 
We can sort the data points according to their loss and for any integer $1 \le r \le n$ we can compute weighted loss discrepancy of $g =\{z_1,\dots,z_r\}$ and $g=\{z_{r+1}, \dots, z_n\}$ to find the group with maximum loss discrepancy.
\end{proof}

Throughout the next proof we use the following statements, 
let $0\le a,b,\hat a, \hat b \le 1$:
\begin{align}
\label{eqn:multi}
|ab - \hat a\hat b| = \left|\frac{(a-\hat a)(b+\hat b) + (a+\hat a)(b-\hat b)}{2}\right|\le |a-\hat a| + |b-\hat b|
\end{align}
\begin{align}
\label{eqn:subtract}
\left| |a-b| - |\hat a-\hat b|\right| \le \left| a-b - \hat a+ \hat b\right| \le |a-\hat a| + |b-\hat b|
\end{align}
Assuming  $\frac{\hat a}{\hat b} \le 1$, the following holds:
\begin{align}
\label{eqn:division}
\pab {\frac{a}{b} - \frac{\hat a}{\hat b}} = \frac{\pab{a\hat b +\hat a \hat b - \hat a \hat b - \hat a b}}{b\hat b} = \frac{\pab {\hat b(\hat a - a)- \hat a(\hat b - b)}}{b\hat b}\le \frac{|a-\hat a| + |a- \hat a|}{b}
\end{align}

\EfficientAuditor*
	
\begin{proof}

Let $D(g)$ denote the weighted loss discrepancy of group $g$, i.e., $D(g) \eqdef  \E [g]^k \pab{\E [\ell \mid g] - \E [\ell]}$	 and analogously $\widehat D(g) \eqdef \widehat\E [g]^k \pab{\widehat\E [\ell \mid g] - \widehat \E [\ell]}$.
Let $g^\star = \argmax_{g\in\sG} D(g)$ and analogously let $\widehat g = \argmax_{g\in\sG} \widehat D(g)$.
We will prove for the following two bounds:

	\begin{align}
	\label{eqn:firstBound}
\pr \pb {D(g^\star) - \widehat D(\widehat g) \ge \epsilon} \le \pr \pb {D(g^\star) - \widehat D(g^\star) \ge \epsilon} \le \frac{\delta}{2}\\
		\label{eqn:secondBound} 
\pr \pb {\widehat D(\widehat g) - D(g^\star) \ge \epsilon } \le \pr \pb {\widehat D(\widehat g) - D(\widehat g) \ge \epsilon} \le \frac{\delta}{2}
	\end{align}
	
We start by proving \refeqn{firstBound}, using \refeqns{multi}{subtract} we have:
\begin{align}
\label{eqn:decomposition}
\nonumber\pr \pb {\pab{D(g^\star) - \widehat D(g^\star)}\ge 3t} \le &\pr \pb {\pab{\E [\ell] -  \widehat\E [\ell]}\ge t}+\\
&\pr \pb {\pab{\E [\ell \mid g^\star=1] -  \widehat\E [\ell \mid g^\star=1]}\ge t}+\\
&\pr \pb {\pab{\E [g^\star]^k - \widehat\E [g^\star]^k}\ge t}
\end{align}

We now compute the terms on the right hand side of the inequality. Using Hoeffding inequality we have:

\begin{align}
\pr \pb {\pab{\E [\ell] -  \widehat\E [\ell]}\ge t} \le 2\exp\p{-2t^2n}
\end{align}

Let $\alpha$ and $\widehat\alpha$ denote $\E [g^\star]$ , $\widehat\E [g^\star]$ respectively.

\begin{align}
\pr \pb {\pab{\E [\ell \mid g^\star=1] -  \widehat\E [\ell \mid g^\star=1]}\ge t}\le& \pr \pb {\widehat \alpha < \alpha/2} + \pr \pb{\pab{\E [\ell \mid g^\star=1] -  \widehat\E [\ell \mid g^\star=1]}\ge t \mid \widehat{\alpha} \ge \alpha/2}\\
\le&\exp\p{-n\alpha/8} + 2\exp \p {-t^2n\alpha}
\end{align}

Let $\beta$ be a number between $\alpha$ and $\widehat \alpha$, formally: $\alpha - |\alpha - \widehat \alpha|\le \beta \le \alpha + |\alpha - \widehat \alpha|$.
Using mean value theorem, we have:

\begin{align}
\label{eqn:meanValue}
\pr \pb {\pab{\alpha^k - \widehat\alpha^k}\ge t} =& \pr \pb {k\beta^{k-1}|\alpha - \widehat \alpha| \ge t}\\
\le & \pr \pb {|\alpha - \widehat{\alpha}| \ge t\beta}\\
\le & \pr \pb {|\alpha - \widehat \alpha| \ge \frac{\alpha}{2}} + \pr \pb {|\alpha - \widehat \alpha| \ge t \beta \mid |\alpha - \widehat\alpha| \le \frac{\alpha}{2}} \\
\le&2\exp\p{-n\alpha/12} + \pr \pb {\pab{\alpha - \widehat{\alpha}}\ge \frac{t}{2}\alpha}\\
\le&2\exp\p{-n\alpha/12} + 2\exp\p{-\frac{t^2n\alpha}{12}}
\end{align}

Combining these three and setting $t= \sqrt{\frac{12\log\p{ 18/\delta}}{n\alpha}}$ and assuming $n \ge \frac{12\log\p { 18/\delta}}{\alpha}$, we have:
\begin{align}
\pr \pb{\pab{D(g^\star) - \widehat D(g^\star)}\ge 3t} \le& 2\exp\p{-2t^2n}+\\
	&\exp\p{-n\alpha/8} + 2\exp\p {-t^2n\alpha}+\\
	& 2\exp(-n\alpha/12) + 2\exp\p{-\frac{t^2n\alpha}{{12}}}\\
	\le & \frac{\delta}{2}
\end{align}

Recall that $\alpha$ denotes $\E [g^\star]$. Since we assumed $0\le \ell \le 1$, we have $D(g^\star) \le \alpha ^k$, therefore, the following holds:

\newcommand{\gbound}{\ensuremath{\p {\frac{108\log (18/\delta)}{n}}^\frac{k}{2k+1}}}
\begin{align}
\label{eqn:final1}
D(g^\star) - \widehat D(g^\star) \le \max_\alpha \min \p {3\sqrt{\frac{12\log\p{ 18/\delta}}{n\alpha}}, \alpha^k} = \gbound
\end{align}

\sidecaptionvpos{figure}{t}
\begin{SCfigure}
	\includegraphics[width=0.4\textwidth]{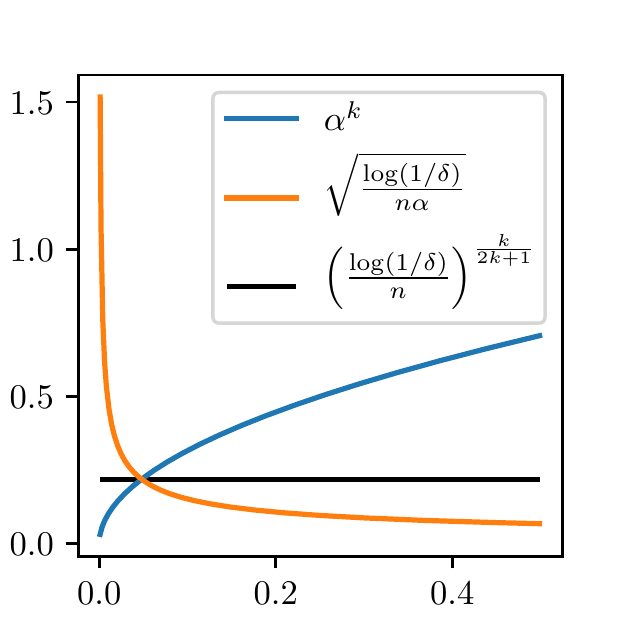}
	\caption{\protect\rule{0ex}{5ex}
	 \label{fig:bound_on_w_k}
	Visualization of \refeqn{final1}.
	Both blue and orange lines are upper bounds for $D(g^\star) -\widehat D(g^\star)$. For small value of $\alpha$, we choose the blue upper bound, and for large value of $\alpha$ we choose the orange upper bound.
	}
\end{SCfigure}
\reffig{bound_on_w_k} shows a visualization of these two upper bounds for $k=\half$.

Note that we use the first upper bound when $\alpha \ge \p {\frac{108\log (18/\delta)}{n}}^\frac{1}{2k+1}$, therefore the constraint on $n$ will be:
\begin{align}
n &\ge \frac{108\log (18/\delta)}{\alpha} \ge \frac{108\log (18/\delta)}{\p {\frac{108\log (18/\delta)}{n}}^\frac{1}{2k+1}} \implies n \ge 108\log (18/\delta)
\end{align}

We now prove Equation~\refeqn{secondBound}.
According to \reflem{computing_U_efficiently}, for a real number $u$, we can represent $\widehat g$ as the group of points with loss less than $u$.
Therefore, it is enough to show the uniform convergence for such groups.
Formally, let $g_u$ be an indicator function  as follows:
\begin{align}
g_u(z)=\begin{cases}
1 & \ell(z) \le u\\ 
0 & o.w.
\end{cases}
\end{align}
We want to prove the following:

\begin{align}
\label{eqn:suprimum}
\pr \pb{ \sup_{u \in [0,1]} \pab { D(g_u) - \widehat D(g_u)} \ge \epsilon } \le \frac{\delta}{2}.
\end{align}
Using Dvoretzky–Kiefer–Wolfowitz inequality \citep{massart1990tight} we have:
\begin{align}
\label{eqn:DKW}
\pr \left[\sup_{x} \left|F(x) - \widehat F (x)\right| \ge \epsilon \right]  \le 2\exp (-2\epsilon^2n)
\end{align}

Similar to the procedure for proving \refeqn{firstBound}, we first prove \refeqn{suprimum} for all groups with $\widehat \E [g_u] \ge \alpha$. 
We then combine the result bound with $\alpha^k$.

As showed in \refeqn{decomposition} we only need to bound the three components.
We represent each component in term of cumulative distribution function ($F(x)$ in \refeqn{DKW}) and then use the same techniques for \refeqn{firstBound} to bound each one of them.

\begin{align}
\pr \pb{ \sup_{u \in [0,1]} \pab { D(g_u) - \widehat D(g_u)} \ge 3t } \le & \pr \pb {\pab {\E [\ell] - \widehat \E [\ell]} \ge t} + \\
&\pr \pb {\sup_{u \in [0,1]} \pab {\E [\ell \mid g_u=1] - \widehat \E [\ell \mid g_u=1]} \ge t}+\\
&\pr \pb {\sup_{u \in [0,1]} \pab {\E [g_u]^k - \widehat \E [g_u]^k} \ge t}
\end{align}

We can rewrite $\E[\ell \mid g_u=1]$ in term of CDF, and bound the second term in the RHS as follows:

\begin{align}
\E [\ell \mid g_u=1] = \frac{\int_0^u xdF(x)}{\int_0^u dF(x)}=\frac{xF(x) |_0^u - \int_0^u F(x)dx}{\int_0^u dF(x)} =\frac{uF(u) - \int_0^u F(x)dx}{F(u)}
\end{align}
\begin{align}
\pab {\E [\ell \mid g_u=1] - \widehat \E [\ell \mid g_u=1]} =& \pab{\frac{uF(u) - \int_0^u F(x)dx}{F(u)} -  \frac{u\widehat F(u) - \int_0^u \widehat F(x)dx}{\widehat F(u)}}\\
\label{eqn:last} \le & \frac{3\pab {F(u) - \widehat F(u)}}{\widehat F(u)}
\end{align}

We derived \refeqn{last} with \refeqn{division}, Recall we assumed $\widehat \E [\ell \mid g_u] = \widehat F(u) \ge \alpha$; therefore, we have:
\begin{align}
 \pr \pb {\sup_{u \in [0,1]}\pab{\widehat\E [\ell \mid g_u=1] - \E [\ell \mid g_u=1]} \ge t} 
  \le&\pr \pb {\sup_{u \in [0,1]}\frac{3\pab{F(u) - \widehat F(u)}}{\widehat F(u)} \ge t} \\
  \le& 2\exp (-2t^2\alpha^2n/9)
\end{align}

Using mean value theorem as explained in \refeqn{meanValue}, the following holds:
\begin{align}
\pab{\E [g_u]^k - \widehat \E [g_u]^k} = \pab {F(u)^k - \widehat F(u)^k} = k \beta^{k-1}\pab {F(u) - \widehat F(u)}
\end{align}

Again Recall we assumed $\widehat F(u) \ge \alpha$:
\begin{align}
\nonumber\pr \pb {\sup_{u \in [0,1]}\pab {\E [g_u]^k - \widehat{\E}[g_u]^k}\ge t} \le& \pr \pb {\pab{F(u) - \hat F(u)} \ge \frac{\alpha}{2}}+\pr \pb {\sup_{u \in [0,1]}\pab {F(u) - \widehat F(u)} \ge \frac{t\alpha}{2}}\\
\le & 2\exp\p{-n\alpha/12} + 2\exp (-t^2\alpha^2n/2)
\end{align}

Combining all three and setting $t=\sqrt{\frac{12\log (18/\delta)}{n\alpha^2}}$  we have:

\begin{align}
\pr \pb {\sup_{u \in [0,1]} \pab{ \widehat D(g_u) - D(g_u)} \ge 3t} \le & 2\exp\p{-2t^2n} + 2\exp (-2t^2\alpha^2n/9) +\\ &2\exp\p{-n\alpha/12} + 2\exp (-t^2\alpha^2n/2) \\
\le & \frac{\delta}{2}
\end{align}

Now combining the above bound with $\alpha^k$, we have:

\begin{align}
\label{eqn:final2}
\sup_{u \in [0,1]} \p{ \widehat D(g_u) - D(g_u)}  \le \max_\alpha \min \p {3\sqrt{\frac{12\log\p{ 18/\delta}}{n\alpha^2}}, \alpha^k} \le {\p {\frac{108\log (18/\delta)}{n}}^\frac{k}{2k+2}}
\end{align}

Combining \refeqn{final1} and \refeqn{final2} completes the proof.

\end{proof}

\section{Datasets}
We divided all datasets using a 70--30 train-test split.
For German dataset, since it has only 1k data points, we split the dataset 10 times and report the average result for this dataset.

\label{app:dataset_appendix}

\begin{table*}
		\newcommand{\scc}{0.75}
	\begin{minipage}{0.3\textwidth}
		\centering
	\scalebox{\scc}{
		\begin{tabular}{lll}
			\toprule
			\multicolumn{3}{c}{\bf Communities and Crime} \\ \midrule
			\bf Black Percentage & Low & High\\
			\bf White Percentage& Low & High\\
			\bf Asian Percentage& Low & High\\
			\bf Hispanic Percentage& Low & High\\
			\bottomrule
		\end{tabular}}
			\end{minipage}\quad%
			\begin{minipage}{0.6\textwidth}
				\centering
		\scalebox{\scc}{
			\begin{tabular}{llllll}
				\toprule
				\multicolumn{6}{c}{\bf Income} \\ \midrule
				\bf Race &  White&  Asian-Pac-Islander & Amer-Indian-Eskimo&  Black & Other\\ 
				\bf Age &  Less than 25& Between 25-45 & Greater than 45&&\\
				\bf Gender & Female & Male & &&\\ 
				\bottomrule
			\end{tabular}}
		\end{minipage}
		
		\caption{\label{tab:sensitiveFeatures} Sensitive attributes ($A$) and the different values they can get for C\&C and Income dataset.}
		\end{table*}


{\bf Communities and Crime} \footnote{\tiny\url{http://archive.ics.uci.edu/ml/datasets/communities+and+crime}}  
dataset represents communities within the United States. 
The data combines socio-economic data from the 1990 US Census, law enforcement data from the 1990 US LEMAS survey, and crime data from the 1995 FBI UCR.
The total number of records is 1994,
and the number of attributes is 122.  
The variable to be predicted is the rate of violent crime in the community.
Following \cite{kearns2018gerrymandering}, we convert the real-valued rate of violent crime to a binary label.
The label of a community is $1$ if that community is among the 70\% of communities with the highest crime rate and $0$ otherwise.
We removed attributes with unknown variables and attributes which tagged as non-predictive in the dataset resulting in 99 attributes. 
We divide data to 70-30 split for train and test.
In this dataset, we chose four continuous attributes related to race as sensitive attributes.
Each one of them indicates the percentage of a specific race in the population.
We discretized these attributes and assigned low if that percentage is less 0.2 and high otherwise.
\reftab{sensitiveFeatures} (left) shows these four attributes and their different values.
Therefore, we have $T=\Pi_i|A_i|=16$ different settings for sensitive attributes and
$2^{\Pi_i|A_i|}-1 = 65535$ sensitive groups.

{\bf Census Income} \footnote{\tiny\url{https://archive.ics.uci.edu/ml/datasets/census+income}}
represents individual attributes and their salary based on census data (also known as "Adult" dataset).
The total number of records is 48842,
and the number of attributes is 14.  
The variable to be predicted is either an individual salary is more than \$50K/yr. 
We removed an attribute named fnlwgt which indicates the number of people the census takers believe that observation represents. 
We used the default 30-70 split provided in the dataset.
We chose the race, age, and gender as sensitive attributes.
We discretized age into three categories: less than 25, between 25 and 45, and greater than 45. 
\reftab{sensitiveFeatures} (right) shows the sensitive attributes and their possible values.
This dataset has $T= \Pi_i |A_i|=30$ different settings for sensitive attributes and $2^{\Pi_i |A_i|} -1 =  2^{30}-1$ sensitive groups.

{\bf  German Credit}\footnote{\tiny\url{https://archive.ics.uci.edu/ml/datasets/statlog+(german+credit+data)}} represents individual attributes and their credit rating (two classes of Good and Bad).
The total number of records is 1000,
and the number of attributes is 20.  
The variable to be predicted is either an individual credit rating is good or bad. 
Since the number of records is small in this dataset, we divide data to 30-70 for train and test 10 times and report the average over these 10 runs.
We chose age and gender as sensitive attributes.
We discretized age same as Income dataset. 
Therefore, this dataset has $T=\Pi_i |A_i| = 6$ and $63$ sensitive groups.

{\bf  COMPAS\_5}
\footnote{\tiny\url{https://www.propublica.org/article/how-we-analyzed-the-compas-recidivism-algorithm}} is a dataset compiled by ProPublica
\footnote{\tiny\url{https://github.com/propublica/compas-analysis}}, it contains a list of criminal offenders screened through the COMPAS (Correctional Offender Management Profiling for Alternative Sanctions) tool in Broward County, Florida during 2013-2014. 
Following \cite{zafar2017fairness} we only used 5 attributes of this dataset: race, gender, age, prior counts, and crime class.
The variable to be predicted is whether individuals recidivated within two years after the screening.  
Considering age, gender, and race as sensitive attributes, the number of sensitive groups in this dataset is as same as Income dataset.

\subsection{Experiment details}
We tune $\eta$ between $\{0.1, 0.01, 0.001, 0.0001\}$ and use grid-search to find the best $\eta$.
In \reffig{showing_efficiency}, we set $\lambda$ in \refeqn{\lv} between $U=\{0, 0.05, 0.1,0.2,0.4,0.6,0.8,1,2,3\}$; and for  $\lambda$ in \refeqn{\clv} we set it to $\{2u \mid u \in U\}$.
In \reftab{group_comparison_compas_5}, we pre-process COMPAS\_5 in a similar fashion to \citet{zafar2017fairness}, only keeping examples with race equal to either black or white and considering race as the only sensitive attribute.
We choose two points on the accuracy-fairness trade-off curve for comparison (one that has accuracy similar to \citet{zafar2017fairness} and one with maximum $\lambda=9$).
\reftab{group_comparison_compas_5} shows the accuracy, \dfn{}, \dfp{} (averaged over 10 runs) for different methods.

\section{Additional figures}
\label{sec:additional_figures}
\begin{figure*}[t]
	\centering
	\begin{subfigure}[b]{0.48\textwidth}
		\centering
		\includegraphics[width=\textwidth]{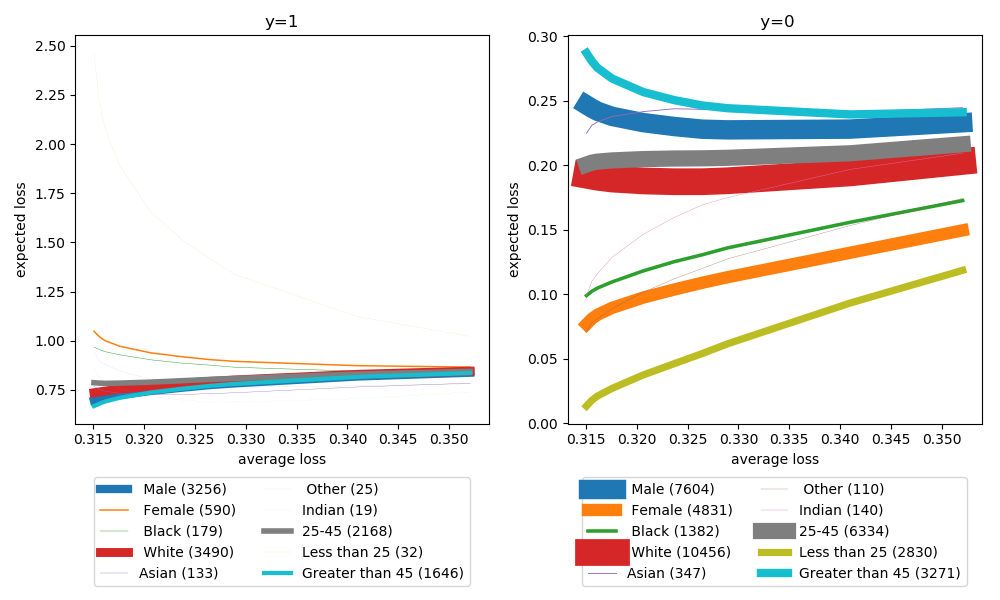}
		\caption{\label{fig:group_visualization}}
	\end{subfigure}\hfill%
	\begin{subfigure}[b]{0.48\textwidth}
		\centering
		\includegraphics[width=\textwidth]{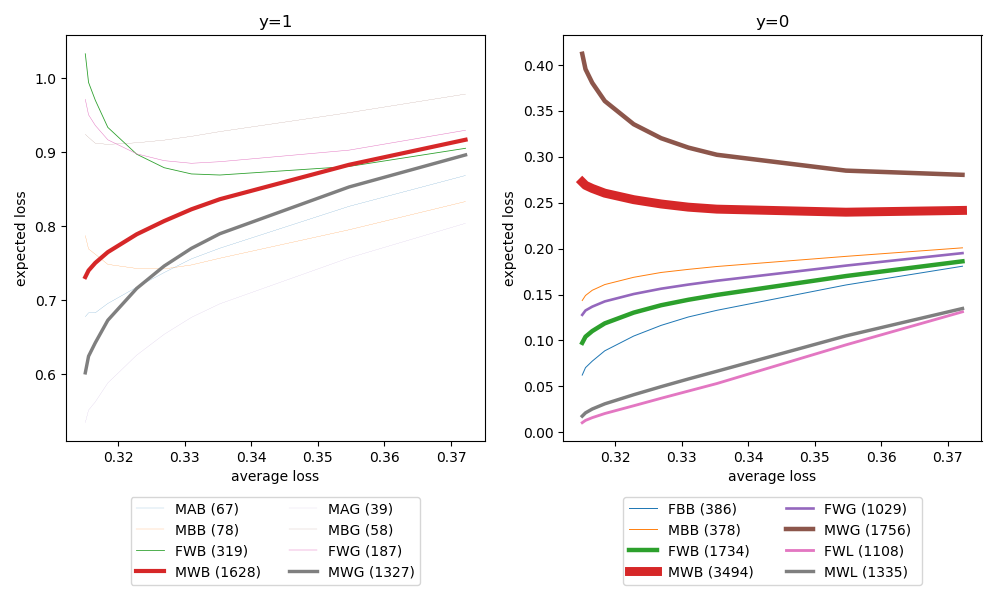}
		\caption{\label{fig:subgroup_visualization} }
	\end{subfigure}
	\caption{\label{fig:group_subgroup_analysis_income} Income dataset. The thickness of the lines is proportional to the size of groups. 
		In the legend, the number in parenthesis indicates the number of data points in that group. 
		In (b) Each group is specified by the first character of its gender, race, and age respectively; e.g., FBB denotes Female-Black-Between 25-45.
		(a) Effect of CLV on sensitive groups defined on a single sensitive attribute.
		(b) 
		Groups with identical sensitive attributes (we only show the top 8 groups with the most data points).
		Coarse loss variance provides a guarantee on any combination of these settings (e.g., individuals who are not White Male or their age is less than 25).}
\end{figure*}
\begin{figure}[t]
	\centering
	\includegraphics[width=0.7\textwidth]{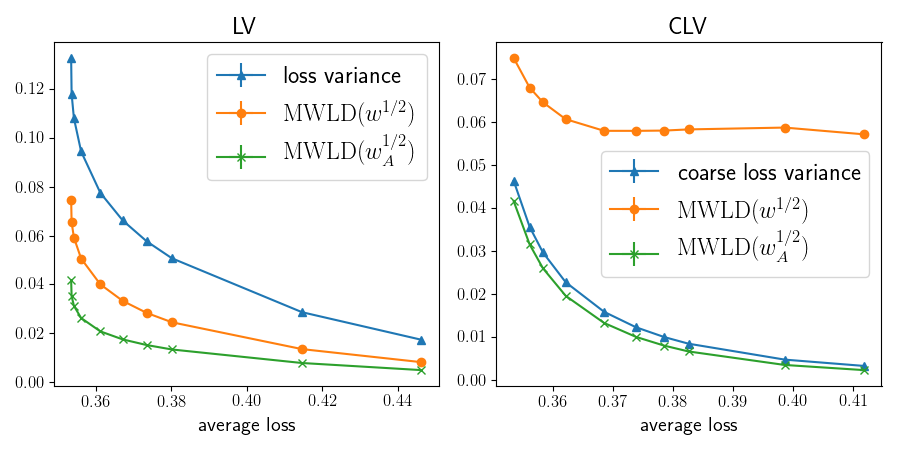}
	
	\caption{\label{fig:group_subgroup_analysis_candc} Effect of (C)LV on maximum weighted loss discrepancy ($\mwld(\ws)$) in C\&C dataset on data points with $y=0$. 
		Left: Loss variance is an upper bound for $\mwld(\ws)$ and consequently $\mwld(\wsa)$.
		Right: Coarse loss variance is a tight upper bound for $\mwld(\wsa)$, but it does not guarantee any upper bound $\mwld(\ws)$.
	}
\end{figure}


We first study the effect of CLV on the losses of groups.
We start with groups defined on a single sensitive attribute.
As shown in \reffig{group_visualization}, in Income dataset, groups such as Males or Whites have a high loss when their label is zero and low loss when their label is one; which means they usually are predicted to have a high paid job even when they have a low paid job.
On the other hand, groups such as Females or Blacks have a high loss when their label is one and a low loss with label zero, which means they usually are predicted to have a low paid job even when they have a high paid job.
As shown, by increasing $\lambda$\refeqn{\clv} losses of these groups become closer to each other.
\reffig{subgroup_visualization} shows the effects of CLV on groups that have identical sensitive attributes.
Note that, Coarse loss variance provides a guarantee on any group consists of any combination of these settings, not only groups defined on a single sensitive attribute. 
For example, we can create a group containing individuals who are not white males or their age is less than 25.


We now study the gap between loss variance and $\mwld(\ws)$ as well as the gap between coarse loss variance and $\mwld(\wsa)$.
Recall that maximum weighted loss discrepancy $\mwld(w)$ (\refdef{mwld}) is the worst-case difference between the loss of a group and the population loss weighted by $w$.
Note that we use conditional loss variance in the experiment; however, for simplicity, here we assume only data points with $y=0$ and compute $\mwld(\ws)$ and $\mwld(\wsa)$ \emph{only} for these data points and drop the conditioning on $y=0$ from loss variance in the sequel.

\reffig{group_subgroup_analysis_candc} shows the value of $\mwld(\ws)$ and $\mwld(\wsa)$ for different values of the $\lambda$ in \refeqn{\lv} and \refeqn{\clv}. 
In LV, loss variance gets penalized which is shown as an upper bound for $\mwld(\ws)$ and subsequently $\mwld(\wsa)$.
In CLV, coarse loss variance gets penalized which is shown as a tight upper bound for $\mwld(\wsa)$; as expected it is not an upper bound for $\mwld(\ws)$.
Note that for reaching the same value of $\mwld(\wsa)$, LV has an average error of $0.45$,
while CLV has an average error of $0.415$, confirming that if coarse loss variance is the quantity of interest,
regularizing with it is better.

\end{document}